\documentclass{article}

\PassOptionsToPackage{numbers,sort&compress}{natbib}
\usepackage[preprint]{neurips_2026}

\usepackage[utf8]{inputenc}
\usepackage[T1]{fontenc}
\usepackage{hyperref}
\usepackage{url}
\usepackage{booktabs}
\usepackage{multirow}
\usepackage[above]{placeins}
\usepackage{amsfonts}
\usepackage{amssymb}
\usepackage{amsmath}
\usepackage{amsthm}
\usepackage{nicefrac}
\usepackage{microtype}
\usepackage{xcolor}
\usepackage{graphicx}
\usepackage{algorithm}
\usepackage{algpseudocode}
\usepackage{float}
\usepackage{wrapfig}
\usepackage{listings}
\usepackage{tcolorbox}
\tcbuselibrary{breakable,skins,listings}

\lstdefinestyle{systemprompt}{
  basicstyle=\ttfamily\scriptsize,
  breaklines=true,
  breakatwhitespace=false,
  columns=fullflexible,
  keepspaces=true,
  showstringspaces=false,
  frame=none,
}

\floatstyle{ruled}
\restylefloat{algorithm}

\algrenewcommand\algorithmicrequire{\textbf{Input:}}
\algrenewcommand\algorithmicensure{\textbf{Output:}}
\algrenewcommand\algorithmicreturn{\textbf{return}}
\algrenewcommand\alglinenumber[1]{#1:}

\title{MIRAGE: Mobile Agents with Implicit Reasoning and Generative World Models}

\author{%
\begin{tabular}{@{}c@{}}
\textbf{Zhichao Yang}$^{1}$ \quad
\textbf{Yuanze Hu}$^{1}$ \quad
\textbf{Haojie Hao}$^{1}$ \quad
\textbf{Longkun Hao}$^{1}$\\[5pt]
\textbf{Dongshuo Huang}$^{2}$ \quad
\textbf{Hongyu Lin}$^{3}$ \quad
\textbf{Gen Li}$^{1}$ \quad
\textbf{Lanqing Hong}$^{4}$\\[5pt]
\textbf{Yihang Lou}$^{5,\dagger}$ \quad
\textbf{Yan Bai}$^{5,\dagger}$\\[6pt]
{\normalfont $^{1}$Beihang University}\\
{\normalfont $^{2}$Northwestern Polytechnical University}\\
{\normalfont $^{3}$Institute of Software, Chinese Academy of Sciences}\\
{\normalfont $^{4}$National University of Singapore}\\
{\normalfont $^{5}$Peking University}\\[2pt]
{\normalfont\small $^{\dagger}$Corresponding authors.}
\end{tabular}%
}

\newcommand{\method}{MIRAGE}

\newcommand{\relgain}[1]{\,{\scriptsize\textcolor{green!50!black}{($\uparrow$#1\%)}}}
\newcommand{\relreduce}[1]{\,{\scriptsize\textcolor{green!50!black}{($\downarrow$#1\%)}}}
\newcommand{\relregress}[1]{\,{\scriptsize\textcolor{red!70!black}{($\uparrow$#1\%)}}}
\newtheorem{lemma}{Lemma}
\newtheorem{proposition}{Proposition}
\newtheorem{corollary}{Corollary}

\begin{document}

\maketitle

\begin{center}
\small Code: \url{https://github.com/yzc0912/MIRAGE}
\end{center}

\begin{abstract}
Mobile agents are increasingly expected to operate everyday applications from screenshots and language goals,
where reliable control requires reasoning over screen affordances, multi-step navigation,
and future state changes. Yet many agents externalize this computation as long textual
thoughts, making interaction slower, supervision more costly, and deployment less
efficient. We introduce \method{} (\textbf{M}obile agents with \textbf{I}mplicit
\textbf{R}easoning \textbf{A}nd \textbf{G}enerative world mod\textbf{E}ls), a
framework that learns continuous latent reasoning representations from visible textual
thoughts. \method{} introduces an efficient latent-space learning procedure that
transfers explicit reasoning into compact hidden states, allowing the agent to reason
internally without decoding long rationales. It further brings a world-model perspective
into mobile-agent training: the model's latent reasoning vectors are aligned with future
screenshots, encouraging the agent to predict upcoming interface states in latent space
before executing an action. This makes the hidden computation not only a compressed
thought trace, but also a forward-looking representation of how the environment may
change. At inference time, \method{} reasons in continuous latent space, reducing token
generation while improving execution efficiency. On AndroidWorld, \method{} matches
explicit-CoT SFT in the 4B ablation under a 3--5$\times$ lower decoded-token budget,
and improves a comparable instruction-tuned baseline by 10.2 points; on
AndroidControl, it improves action grounding with over 75\% fewer generated tokens.
\end{abstract}

\section{Introduction}

As vision-language models improve, more mobile-agent systems now use them to execute
mobile operations directly from screenshots and user instructions. Recent systems such
as UI-TARS, MAI-UI, OS-ATLAS, and SeeClick train VLMs to understand GUI screens and
produce taps, swipes, text entry, or navigation commands
\citep{qin2025ui,zhou2025mai,wu2024atlas,cheng2024seeclick}. Yet turning
a screen-level VLM into a reliable mobile agent still hinges on navigation reasoning:
the model must track task progress, decide which screen to visit next, and anticipate
how each action will change the interface. Figure~\ref{fig:intro} highlights this
reasoning bottleneck.
Current mobile agents often make this reasoning explicit through long thoughts or
verbose action traces, which increases decoding time, context usage, and supervision
cost. Their execution accuracy remains limited, suggesting that mobile agents need
reasoning that preserves explicit-trace capability while being cheaper to run under
realistic deployment constraints, where every extra token slows the interaction and
delays feedback during interactive control.

\begin{figure}[t]
  \centering
  \includegraphics[width=\linewidth]{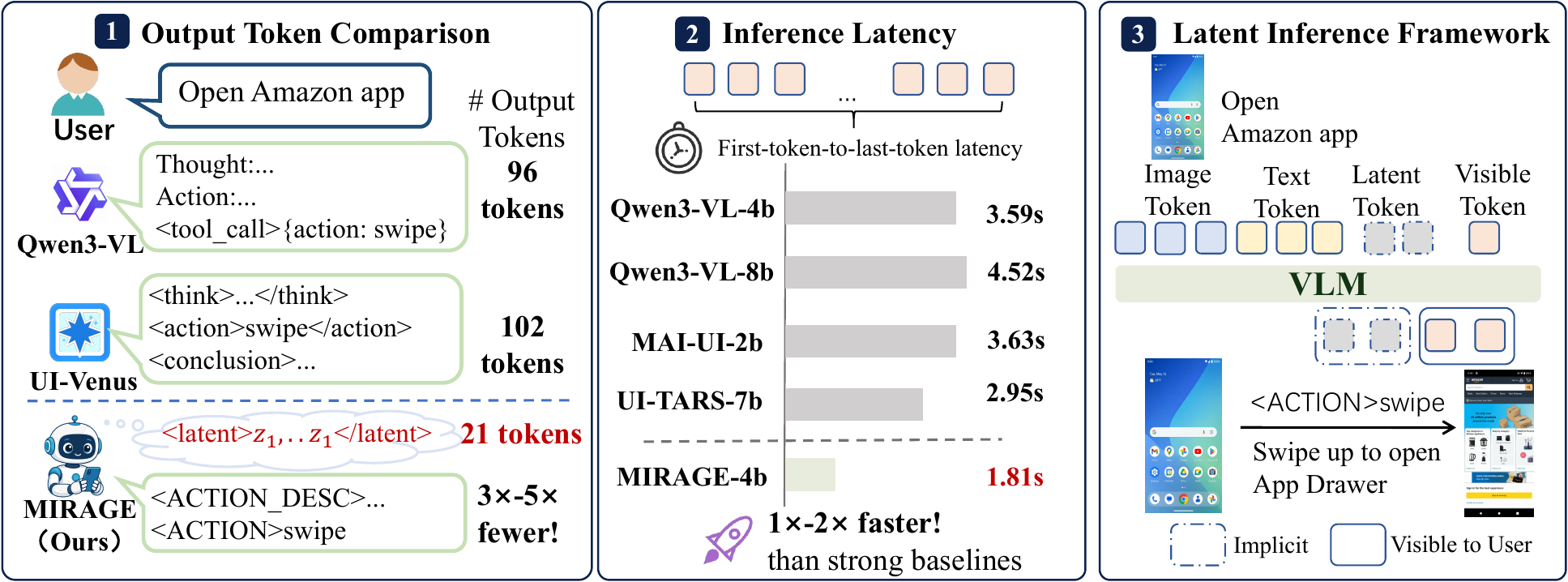}
  \caption{Inference-time comparison on a randomly sampled task and \method{} workflow:
  baselines emit verbose visible traces, while \method{} reasons with latent tokens and
  exposes concise actions.}
  \label{fig:intro}
\end{figure}

To address these issues, we propose \method{}, whose central idea is that a mobile agent should learn to
``think forward'' inside the model. Before emitting an action, the agent forms a
compact internal representation of what the current screen affords, why an action is
appropriate, and how the interface is likely to change afterwards. This reasoning is
carried out in latent space rather than decoded as visible text, so \method{} avoids
emitting intermediate thoughts and reduces both output tokens and
first-token-to-last-token latency. \method{} uses a two-stage training procedure. In the
first stage, it trains on explicit text traces so the backbone learns the mobile action
space and how to express observation, rationale, and future-screen prediction. In the
second stage, \method{} replaces the textual reasoning block with continuous latent
reasoning slots and continues training, gradually teaching the model to move reasoning
into latent space. At inference time, only action tokens are decoded; no rationale
text is emitted and the interaction latency is substantially reduced.

To transfer explicit reasoning capability into latent space, \method{} introduces
Approximate Parallel Latent Refinement~(APLR), which refines all latent slots in
parallel through $K$ Jacobi-style rounds, approximating full sequential latent
chain-of-thought: the first $K$ slots provably match the serial rollout while the
remaining tail error is bounded. To further enable the model to internalize future-state
prediction, \method{} attaches a lightweight Q-Former world-model head that aligns the
output latent representations with stop-gradient visual features of the next screenshot
from the frozen vision encoder. This alignment teaches the agent to anticipate GUI state
transitions, prevents latent representation collapse, and supplies dense supervision that
compensates for APLR's bounded tail error. Crucially, the action cross-entropy loss and
the next-frame feature-alignment loss are optimized jointly, so the same latent states
become both action-discriminative and transition-predictive. This coupling lets
\method{} retain CoT-level reasoning capacity while moving the intermediate computation
out of the decoded text stream, yielding comparable explicit-reasoning capability at
substantially lower inference cost.

\method{} yields strong results on established mobile-agent benchmarks. On AndroidWorld,
\method{} improves task success by 10.2 percentage points over a comparable
instruction-tuned baseline and matches explicit CoT in the 4B ablation under a
3--5$\times$ lower decoded-token cost. On
AndroidControl, it improves action-grounding accuracy while reducing generated tokens
by over 75\%, demonstrating that latent reasoning and future-state alignment preserve
explicit-reasoning capability with substantially leaner inference.

This paper makes three contributions:
\begin{itemize}
    \item We present \method{}, a mobile agent that reasons entirely in latent space,
      producing far fewer output tokens and substantially lower inference latency
      while matching explicit-CoT task performance in the 4B ablation under a
      3--5$\times$ lower decoded-token cost.
    \item We introduce APLR, a parallel Jacobi-style latent refinement procedure that
      approximates full serial latent chain-of-thought at a fraction of the training
      cost, with a provable bound on the approximation error of the tail slots.
    \item We introduce a Q-Former world-model head that aligns latent reasoning states
      with future-screenshot features in latent space, enabling the agent to predict
      upcoming GUI transitions and directly improving task capability.
\end{itemize}

\section{Related Work}

\paragraph{Mobile and GUI agents.}
GUI-agent benchmarks have evolved from web interaction to dynamic Android control, covering grounded shopping, open-web navigation, data-scaling studies, and online task completion \citep{yao2022webshop,deng2023mind2web,li2024effects,rawles2024androidworld}. Recent systems specialize VLMs for screenshot grounding, GUI actions, and mobile-device operation \citep{cheng2024seeclick,wu2024atlas,qin2025ui,wang2024mobile}. Instead of extending grounding or planning pipelines, \method{} trains internal agent states by replacing visible reasoning traces with latent slots predictive of the next GUI state during action decoding and transition modeling.

\paragraph{Reasoning in language and vision-language agents.}
Visible CoT and ReAct-style traces improve reasoning and acting but expose long rationales and consume context \citep{wei2022chain,yao2022react}. Other work internalizes computation through pause tokens, private thoughts, distillation, or continuous latent CoT \citep{goyal2023think,zelikman2024quiet,deng2023implicit,hao2022training}. \method{} brings implicit reasoning to mobile GUI control, where latent thoughts support action selection and transition understanding, using APLR to approximate serial latent refinement in a strictly causal triangular system rather than an equilibrium model.

\paragraph{World models and visual feature prediction.}
World models learn predictive dynamics for control, from compact latent simulators to latent-imagination agents \citep{ha2018world,hafner2019dream}. Joint-embedding vision objectives show that feature prediction can learn semantics without pixel generation \citep{assran2023self,bardes2024revisiting}, while modal-alignment work highlights bottlenecks between visual and language representations \citep{hu2025tinyalign}. Q-Former queries offer a lightweight cross-attention bottleneck \citep{li2023blip}. Recent GUI world models predict future screens, sketches, or semantic states \citep{luo2025vimo,cao2026mobiledreamer}. \method{} instead uses future prediction to shape latent reasoning states, encouraging action-induced transition representations without generating pixels or future text at inference time.

\section{Method}

\subsection{Problem Formulation: From Explicit-Thought to Latent-Thought Mobile Agents}

At interaction step $t$, a mobile GUI agent observes the current screenshot $x_t$, a user instruction $u$, and an interaction history $h_t=(a_{<t},x_{<t})$. It outputs an action $a_t$, such as tapping a coordinate, typing text, scrolling, or navigating back. We serialize actions as text tokens so that a VLM can model mobile control as conditional generation:
\begin{equation}
    o_t = (x_t,u,h_t),
    \qquad
    p_\theta(a_t \mid o_t).
    \label{eq:mobile_action}
\end{equation}

Mobile GUI agents face a critical trade-off between reasoning capability and
deployment efficiency. Explicit chain-of-thought improves multi-step decision
making because the model can verbalize observations, rationales, and predicted
state transitions before acting; however, this reasoning is paid for with visible
tokens. Many explicit-thought mobile agents first generate a readable reasoning
trace $\tau_t$ and then generate the action. This factorization can be written as
\begin{equation}
    p_\theta(\tau_t,a_t \mid o_t)
    =
    \prod_{i=1}^{|\tau_t|}
    p_\theta\!\left(\tau_t^{(i)} \mid \tau_t^{(<i)},o_t\right)
    \cdot
    \prod_{j=1}^{|a_t|}
    p_\theta\!\left(a_t^{(j)} \mid \tau_t,a_t^{(<j)},o_t\right).
    \label{eq:explicit_cot_agent}
\end{equation}
The intermediate thought $\tau_t$ is useful, but it is also a sequence of
discrete text tokens: it consumes context budget, increases decoding latency, and
usually requires human-written or synthetic rationale supervision. This makes
explicit CoT attractive as a training scaffold but costly as an inference-time
interface for mobile deployment.

To resolve this tension, \method{} formulates mobile control as a latent-thought
generation problem. Instead of removing reasoning, we internalize it: the model
keeps a structured computation budget, but the intermediate computation is carried
by hidden states rather than text. We introduce $N$ continuous latent variables
$\mathbf{z}_t=(z_{t,1},\ldots,z_{t,N})$, $z_{t,i}\in\mathbb{R}^d$, in the decoder context:
\begin{equation}
    p_\theta(a_t \mid o_t)
    =
    \int
    p_\theta(a_t \mid \mathbf{z}_t,o_t)
    q_\theta(\mathbf{z}_t \mid o_t)
    d\mathbf{z}_t .
    \label{eq:latent_cot_agent}
\end{equation}
Unlike explicit CoT, $\mathbf{z}_t$ does not live in the vocabulary and is never
detokenized. Its inference distribution $q_\theta$ is computed implicitly by
decoder hidden states, so latent-thought training does not need rationale-text
supervision. This formulation gives four practical advantages: it preserves
context by removing explicit rationale tokens, speeds up inference by shortening
the decoded sequence, retains multi-step reasoning capacity through latent slots,
and exposes the hidden reasoning state to auxiliary regularization such as
next-frame prediction. In this sense, latent thoughts bridge the gap between
CoT-style reasoning strength and efficient mobile-agent deployment, motivating the
Approximate Parallel Latent Refinement(APLR) refinement and world-model regularization developed next.

\begin{figure}[t]
  \centering
  \includegraphics[width=\linewidth]{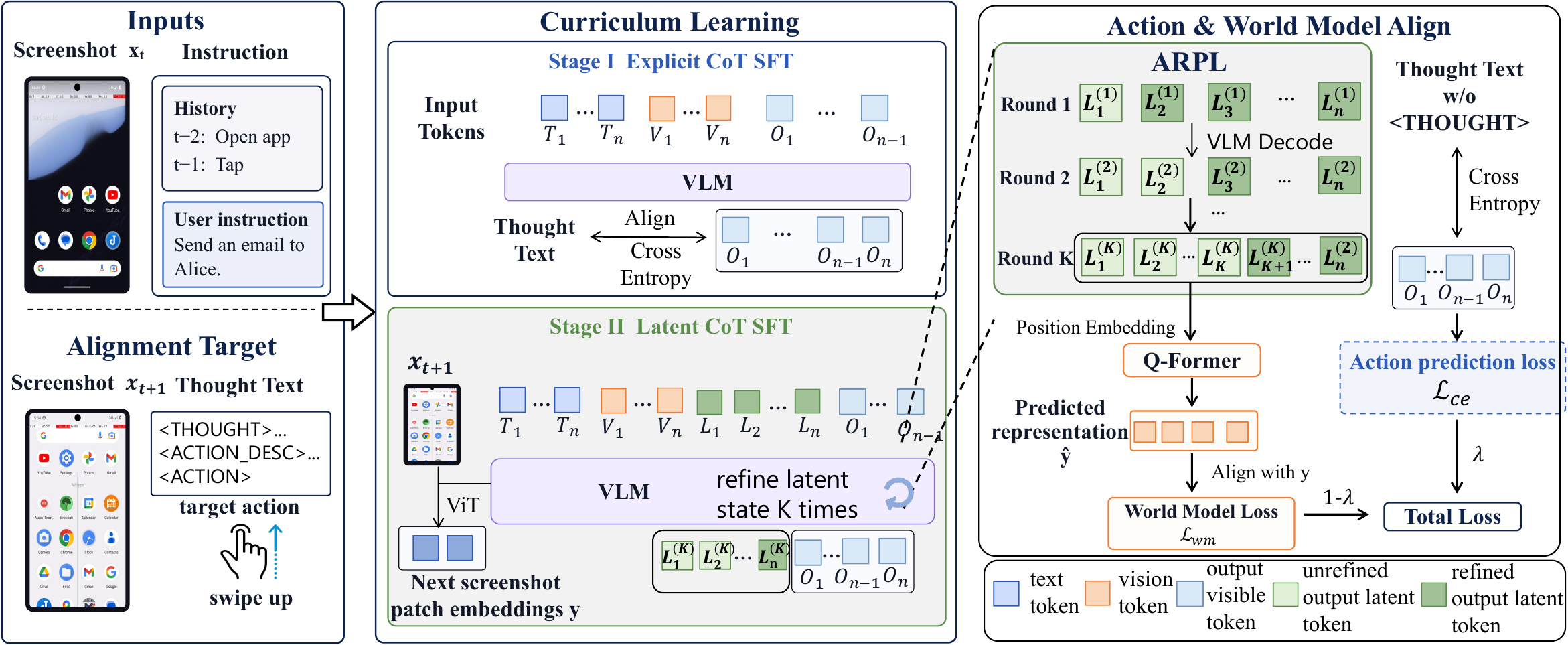}
  \caption{\method{} pipeline. Stage~1 learns explicit mobile thoughts and action formatting. Stage~2 replaces the thought text with latent slots, refines them with APLR, and trains a Q-Former world model to align latent states with next-frame visual features.}
  \label{fig:pipeline}
\end{figure}

\subsection{Training Latent Chain-of-Thought via Approximate Parallel Latent Refinement}

\subsubsection{Structured Thought Data and Latent Insertion}

Our training data uses a structured thought format with three mobile-agent reasoning dimensions:
\begin{quote}
\small\ttfamily\raggedright
<THOUGHT> [observation] [rationale] [predict] <ACTION\_DESC> ... <ACTION> ...
\end{quote}
The \texttt{observation} field describes the visible screen state, \texttt{rationale} explains why a particular operation should be taken, and \texttt{predict} describes the expected next-screen transition. Explicit-thought warmup uses this structured text as ordinary next-token supervision. Latent-thought training then replaces the entire \texttt{<THOUGHT>} block with $N$ latent tokens.
The token sequence becomes
\begin{equation}
    [\mathrm{ctx}]\ ;
    [\mathrm{start}]\ ;
    \underbrace{\langle\mathrm{lat}\rangle,\ldots,\langle\mathrm{lat}\rangle}_{N\ \mathrm{slots}}\ ;
    [\mathrm{end}]\ ;
    \langle\mathrm{ACTION\_DESC}\rangle \cdots
    \langle\mathrm{ACTION}\rangle .
    \label{eq:latent_sequence}
\end{equation}
The latent slots occupy positions $p_1<\cdots<p_N$ and share the learned initializer $e_{\mathrm{lat}}$.

\subsubsection{From Serial Masking to Approximate Parallel Latent Refinement}
\label{sec:aplr}

Original Coconut-style latent reasoning \citep{hao2022training} replaces visible reasoning steps with continuous hidden states and computes those states serially; Appendix~\ref{app:original_coconut} gives the full formalization. Our APLR keeps the same latent-thinking motivation but changes the execution pattern: rather than advancing through thought steps one by one, it treats the thought block as $N$ slots and refines all slots in parallel rounds.

With APLR, latent reasoning becomes a small number of synchronous refinement passes
instead of a long serial rollout. Conceptually, this lets the model keep an internal
multi-step computation budget while making training cost depend on the refinement
depth $K$, not directly on the latent length $N$.

With fixed non-latent context $c=o_t$ and latent positions $p_1<\cdots<p_N$, the serial target states are
\begin{equation}
    s_i
    =
    G_i(s_1,\ldots,s_{i-1};c),
    \qquad i=1,\ldots,N ,
    \label{eq:serial_coconut}
\end{equation}
where $G_i$ is the hidden state at position $p_i-1$ when the previous latent states have already been written into the sequence. This is a forward-substitution procedure over a causal triangular system. It is accurate, but expensive: fully refining all $N$ latent tokens requires one full decoder pass per latent update, plus a final pass for action logits.

We use APLR to eliminate the need for $N$ serial forward passes in traditional latent CoT training. The key observation is that causal attention makes the latent system strictly triangular: $z_i$ cannot depend on itself or on any future $z_{j>i}$. Therefore, instead of filling the slots one by one, APLR performs Jacobi-style rounds:
\begin{equation}
    z_i^{(k+1)}
    =
    G_i\!\left(z_1^{(k)},z_2^{(k)},\ldots,z_{i-1}^{(k)};c\right),
    \qquad
    i=1,\ldots,N,\quad k=0,\ldots,K-1 .
    \label{eq:aplr_update}
\end{equation}
All right-hand sides use old round-$k$ values, so one full forward updates all slots. In practice, we use a small refinement budget, $K=3$ by default.

APLR has a precise relationship to the serial refinement in \eqref{eq:serial_coconut}. After $K$ parallel rounds, the first $K$ latent slots exactly match the serial solution, $z_i^{(K)}=s_i$ for $i\leq K$. Compared with refining all latent tokens one by one, the unrefined tail slots retain a structured residual
\begin{equation}
    \delta^{(K)} \approx A^K\delta^{(0)},
    \label{eq:aplr_tail_error}
\end{equation}
where $A$ is a strictly lower-triangular Jacobian. Appendix~\ref{app:aplr_proof} proves that after $K$ APLR rounds, the first $K$ latent slots exactly recover the serial solution, and then derives the tail-error form in \eqref{eq:aplr_tail_error}. In implementation, early bootstrap passes may run without gradient, while the first pass can retain gradient through $e_{\mathrm{lat}}$ for stability. Before the final gradient-enabled pass, we rebuild the input embeddings with a mask: latent positions use detached bootstrap values, while non-latent positions retain their original gradient lineage to the vision tower and token embeddings.

The final APLR pass is trained with standard next-token CE over serialized action tokens:
\begin{equation}
    \mathcal{L}_{\mathrm{ce}}
    =
    -\sum_{j=1}^{|a_t|}
    \log
    p_\theta\!\left(
    a_t^{(j)}
    \mid
    \mathbf{z}_t^{(K)},a_t^{(<j)},o_t
    \right).
    \label{eq:action_ce}
\end{equation}
When $K\ll N$, the tail latents $z_{K+1:N}^{(K)}$ retain approximation error relative to a trainer that refines every latent token serially. In our structured thought format, the \texttt{predict} dimension concerns how the screen will change after the action and typically lies late in the thought sequence. Action CE only penalizes tail-error directions that affect the next action tokens; it does not directly supervise errors that matter for environment prediction. This motivates the world-model objective.

\subsection{World Model: Q-Former Next-Frame Alignment}

The \texttt{predict} field describes future screen semantics after a mobile action. Unlike GUI world models that generate pixels or sketches~\citep{luo2025vimo}, our Q-Former head directly regularizes latent reasoning states to predict future GUI features, forcing latent slots to encode environment dynamics rather than only current-screen action correlations. Tail-error directions that harm future-screen prediction therefore receive direct gradients; Appendix~\ref{app:wm_tail_error} formalizes this with a local second-order argument.

Because the training records are sequential, most non-terminal step-$t$ examples provide the next screenshot $x_{t+1}$. We use it as feature supervision, not as a pixel-generation target: $x_{t+1}$ is passed through the VLM's own frozen vision tower, yielding a stable next-frame feature target that does not drift with the world-model head.

We gather the final-pass hidden states at latent positions to obtain latent CoT hidden states
\begin{equation}
    C_t = \mathrm{Gather}\!\left(H_t^{(K)},p_1,\ldots,p_N\right)
    \in \mathbb{R}^{L\times H},
    \label{eq:cot_memory}
\end{equation}
where each row is one latent thought vector after APLR. The next screenshot is first converted by the VLM image processor into visual patches. After the VLM's spatial merge operation, each target patch has both a feature vector and a 2D grid coordinate. We write the target features as
\begin{equation}
    V_{t+1}
    =
    \mathrm{sg}\!\left(f_{\mathrm{vis}}(x_{t+1})\right)
    =
    [v_1,\ldots,v_M]
    \in \mathbb{R}^{M\times H_v},
    \label{eq:next_feature_target}
\end{equation}
where $\mathrm{sg}(\cdot)$ denotes stop-gradient. The stop-gradient prevents collapse through the target branch, no image decoder is required, and the target lives on the same visual manifold as the backbone features.

For each next-frame patch $j$ with grid coordinate $(r_j,c_j)$, we build a learnable spatial query
\begin{equation}
    q_j = e^{\mathrm{row}}_{r_j} + e^{\mathrm{col}}_{c_j}
    \in \mathbb{R}^{d_q}.
    \label{eq:qformer_query}
\end{equation}
This says, in simple terms, ``predict the feature at row $r_j$ and column $c_j$ of the next screen.'' The separable row/column design gives a stable 2D positional prior even when mobile screenshots have different resolutions and therefore different patch grids.

The BLIP-2-style Q-Former aligner \citep{li2023blip} self-attends over queries, cross-attends to $C_t$, and projects each query output into the VLM feature space:
\begin{equation}
    U_t = \mathrm{QFormer}_{\phi}(Q_t,C_t),
    \qquad
    \hat{V}_{t+1}=W_{\mathrm{vis}}U_t .
    \label{eq:qformer_forward}
\end{equation}
The output $\hat{V}_{t+1}=[\hat{v}_1,\ldots,\hat{v}_M]$ is therefore a predicted next-frame representation, one vector per next-frame patch. The default objective is masked per-patch cosine distance:
\begin{equation}
    \mathcal{L}_{\mathrm{wm}}
    =
    \frac{1}{|\mathcal{M}|}
    \sum_{(b,j)\in\mathcal{M}}
    \left(
    1-
    \frac{
    \langle \hat{v}_{b,j},v_{b,j}\rangle
    }{
    \|\hat{v}_{b,j}\|_2\|v_{b,j}\|_2
    }
    \right).
    \label{eq:wm_cosine}
\end{equation}
The mask $\mathcal{M}$ keeps only valid patches from non-terminal latent samples; cosine is default, with MSE used for ablations.

\subsection{Two-Stage Training Pipeline}

We use a curriculum learning strategy to make latent reasoning learnable rather
than asking the model to discover a continuous thought space from scratch. Stage~1
warm-ups the model on explicit structured thoughts to learn action formatting and
observation--rationale--prediction reasoning patterns. Specifically, the full
\texttt{<THOUGHT>} trace is exposed as text supervision and the VLM is trained with
standard next-token CE over the structured thought, action description, and action
tokens. Stage~2 distills this explicit reasoning process into compact latent slots:
the \texttt{<THOUGHT>} block is replaced by the latent sequence in
\eqref{eq:latent_sequence}, APLR refines the latent states ($K=3$ by default), the
collator loads next screenshots, and the world model regularizes the latent slots
through Q-Former next-frame feature alignment. The joint objective is
\begin{equation}
    \mathcal{L}
    =
    \lambda \mathcal{L}_{\mathrm{ce}}
    +
    (1-\lambda)\mathcal{L}_{\mathrm{wm}},
    \qquad
    \lambda \in (0,1).
    \label{eq:joint_loss}
\end{equation}
We use $\lambda=0.8$ by default. Figure~\ref{fig:pipeline} summarizes the pipeline, and Appendix~\ref{app:pipeline_algorithm} gives pseudocode. At inference time, the agent only performs latent substitution and greedy action decoding; the Q-Former head is used only for training-time representation shaping \citep{yuan2026fast}.

\section{Experiments}

\subsection{Setup}

\paragraph{Backbones.}
We fine-tune two vision-language models from the Qwen3-VL family~\citep{bai2025qwen3}:
\textbf{Qwen3-VL-4B-Instruct} (4B parameters) and \textbf{Qwen3-VL-8B-Instruct} (8B parameters).
Both backbones share the same action-serialization vocabulary, latent-slot token format,
and Q-Former world-model head. Their latent computation budgets differ: the main
Qwen3-VL-4B setting uses 9 latent slots with 3 APLR refinement passes, while the
Qwen3-VL-8B setting uses 6 latent slots with 3 refinement passes.

\paragraph{Evaluation benchmarks.}
We evaluate on two standard mobile-agent benchmarks.
\textbf{AndroidControl}~\citep{li2024effects} provides paired high-level and
low-level instructions with ground-truth action sequences, allowing separate measurement of
instruction-following exact match (EM) and action accuracy.
\textbf{AndroidWorld}~\citep{rawles2024androidworld} is a dynamic, on-device benchmark
spanning 116 real-world task instances across 20 Android apps, measuring end-to-end
task completion rate under live Android dynamics.

\paragraph{Baselines.}
We compare against size-matched \textbf{Qwen3-VL-4B/8B-Instruct} backbones
\citep{bai2025qwen3}, a general multimodal baseline \textbf{GPT-4o}
\citep{hurst2024gpt}, reinforcement-tuned GUI agents \textbf{GUI-R1}/\textbf{UI-R1}
\citep{luo2025gui,lu2026ui}, and recent GUI-agent systems including
\textbf{ShowUI}, \textbf{MAI-UI}, \textbf{UI-Venus-Navi}, \textbf{UI-TARS-7B-SFT},
and \textbf{Ferret-UI Lite} across AndroidControl and AndroidWorld
\citep{lin2025showui,zhou2025mai,gu2025ui,qin2025ui,yang2025ferret}.

\subsection{Main Results}

\begin{table}[t]
  \caption{AndroidControl results. EM = exact match, Action Acc. = action
  accuracy, and Tokens = average generated tokens per step. For the literature-reported
  GPT-4o, GUI-R1, and UI-R1 rows, EM and Action Acc. contain the published step
  success rate and action-type accuracy, respectively~\citep{luo2025gui}.
  Green parentheses show \method{} gains over
  size-matched Qwen3-VL-Instruct baselines. Best primary metrics and lowest
  Tokens are in \textbf{bold}.}
  \label{tab:androidcontrol}
  \centering
  \small
  \setlength{\tabcolsep}{3.5pt}
  \resizebox{\linewidth}{!}{%
  \begin{tabular}{@{}l c ccc ccc@{}}
    \toprule
    \multirow{2}{*}{\textbf{Model}} & \multirow{2}{*}{\textbf{Size}}
      & \multicolumn{3}{c}{\textbf{Low-Level}}
      & \multicolumn{3}{c}{\textbf{High-Level}} \\
    \cmidrule(lr){3-5}\cmidrule(lr){6-8}
      & & EM & Action Acc. & Tokens & EM & Action Acc. & Tokens \\
    \midrule
    GPT-4o & -- & 28.39 & 74.33 & -- & 21.17 & 63.06 & -- \\
    \midrule
    GUI-R1-3B & 3B & 64.41 & 83.68 & -- & 46.55 & 58.04 & -- \\
    UI-R1-3B & 3B & 66.44 & 79.15 & -- & 45.44 & 57.85 & -- \\
    GUI-R1-7B & 7B & 66.52 & 85.17 & -- & 51.67 & 71.63 & -- \\
    ShowUI-2B & 2B & 74.83 & 81.68 & 31.57 & 59.82 & 74.18 & 34.39 \\
    MAI-UI-2B & 2B & 64.70 & 74.17 & 32.87 & 64.30 & 67.3 & 32.78 \\
    UI-Venus-Navi-7B & 7B & \textbf{85.09} & 93.05 & 103.66 & \textbf{75.2} & \textbf{86.14} & 96.36 \\
    UI-TARS-7B-SFT & 7B & 76.28 & 82.16 & 91.80 & 62.12 & 74.37 & 91.38 \\
    \midrule
    Qwen3-VL-4B-Instruct & 4B & 68.48 & 75.15 & 115.67 & 54.58 & 69.39 & 116.00 \\
    Qwen3-VL-8B-Instruct & 8B & 77.66 & 82.54 & 79.86 & 60.78 & 74.66 & 80.52 \\
    \midrule
    \method{}-4B & 4B & 77.59\relgain{13.30} & 91.09\relgain{21.21} & 18.92\relreduce{83.64} & 64.43\relgain{18.05} & 75.67\relgain{9.05} & 20.61\relreduce{82.23} \\
    \method{}-8B & 8B & 83.75\relgain{7.84} & \textbf{94.62}\relgain{14.64} & \textbf{18.01}\relreduce{77.45} & 72.35\relgain{19.04} & 82.83\relgain{10.94} & \textbf{19.89}\relreduce{75.30} \\
    \bottomrule
  \end{tabular}%
  }
\end{table}

\begin{table}[t]
  \caption{AndroidWorld results. SR = task success rate; Avg.\ Steps are
  per task, and Avg.\ Tokens are per step. Green parentheses show \method{} gains over size-matched
  Qwen3-VL-Instruct baselines; red marks higher step count. Best SR and lowest
  averages are in \textbf{bold}.}
  \label{tab:androidworld}
  \centering
  \small
  \begin{tabular*}{\linewidth}{@{\extracolsep{\fill}}l c c c c}
    \toprule
    \textbf{Model} & \textbf{Size} & \textbf{SR} & \textbf{Avg.\ Steps} & \textbf{Avg.\ Tokens} \\
    \midrule
    GPT-4o & -- & 30.6 & -- & -- \\
    \midrule
    Ferret-UI Lite & 3B & 28.0 & -- & -- \\
    UI-TARS & 7B & 33.0 & -- & -- \\
    ShowUI-2B & 2B & 7.8 & 18.9 & 33.0 \\
    MAI-UI-2B & 2B & 46.5 & 15.0 & 116.0 \\
    UI-Venus-Navi-7B & 7B & 42.2 & 13.3 & 92.0 \\
    UI-TARS-7B-SFT & 7B & 33.6 & 17.3 & 104.0 \\
    \midrule
    Qwen3-VL-4B-Instruct & 4B & 42.9 & 14.3 & 103.0 \\
    Qwen3-VL-8B-Instruct & 8B & 47.6 & \textbf{12.6} & 108.0 \\
    \midrule
    \method{}-4B & 4B & 52.6\relgain{22.6} & 14.2\relreduce{0.7} & 31.0\relreduce{69.9} \\
    \method{}-8B & 8B & \textbf{57.8}\relgain{21.3} & 13.7\relregress{8.7} & \textbf{27.0}\relreduce{75.0} \\
    \bottomrule
  \end{tabular*}
\end{table}

\paragraph{AndroidControl.}
Table~\ref{tab:androidcontrol} shows that \method{} improves the size-matched
Qwen3-VL-Instruct baselines while producing much shorter action outputs. On the
low-level split, \method{}-4B improves EM from 68.48 to 77.59 and action accuracy from
75.15 to 91.09, while reducing the average generation length from 115.67 to
18.92 tokens. \method{}-8B shows the same pattern, increasing low-level EM from
77.66 to 83.75 and action accuracy from 82.54 to 94.62 with 18.01 tokens per step.
On the high-level split, \method{}-4B improves EM/action accuracy by 9.85/6.28
percentage points over Qwen3-VL-4B-Instruct, and \method{}-8B improves them by
11.57/8.17 points over Qwen3-VL-8B-Instruct. These gains suggest that replacing
verbose visible thoughts with latent reasoning can improve grounding and action
selection without increasing the decoded token budget.

\paragraph{AndroidWorld.}
Table~\ref{tab:androidworld} evaluates the same models in a dynamic on-device
setting. \method{} raises AndroidWorld SR from 42.9 to 52.6 for 4B and from
47.6 to 57.8 for 8B, while reducing average tokens from 103.0/108.0 to
31.0/27.0. The 8B model uses slightly more steps than its Qwen3-VL baseline,
so the gain does not simply come from shorter trajectories. Among retained
specialized GUI agents, \method{}-8B gives the highest AndroidWorld SR with far
fewer generated tokens. Together, the two benchmarks show that latent reasoning
and world-model training improve task effectiveness while keeping outputs compact.

Figure~\ref{fig:latency_and_breakdown} further analyzes efficiency and robustness.
The left panel measures first-to-last-token latency; \method{}-4B is fastest,
consistent with the token reductions in Tables~\ref{tab:androidcontrol}
and~\ref{tab:androidworld}. The right panel breaks AndroidControl low-level results
into IDD, app-unseen, category-unseen, and task-unseen subsplits.
\method{}-8B achieves the highest EM in each subsplit, and both \method{} variants
achieve higher action accuracy than their size-matched backbones across all four
subsplits.

\begin{figure}[t]
  \centering
  \begin{minipage}[t]{0.46\linewidth}
    \centering
    \includegraphics[width=\linewidth]{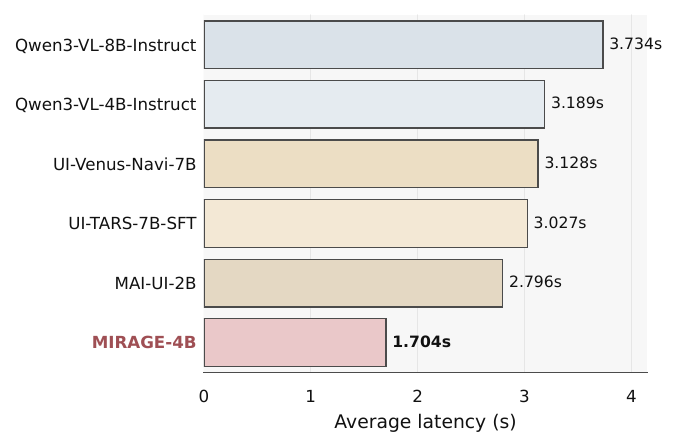}
  \end{minipage}\hfill
  \begin{minipage}[t]{0.50\linewidth}
    \centering
    \includegraphics[width=\linewidth]{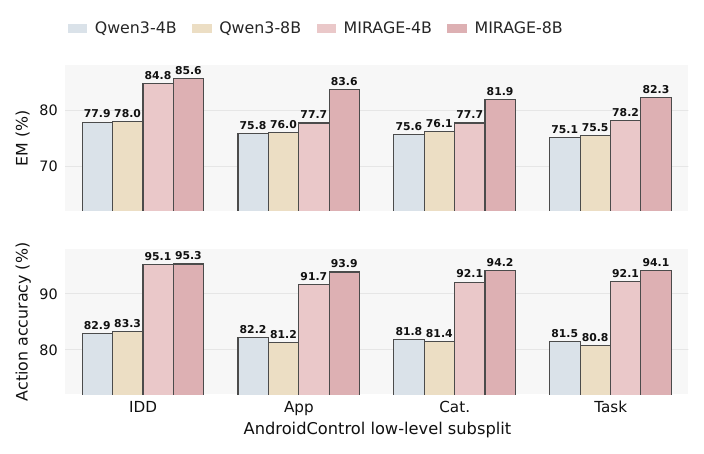}
  \end{minipage}
  \caption{\textbf{Left:} Average latency from the first generated token to the final
  generated token. \method{}-4B produces the shortest decoded sequence latency among
  the compared models. \textbf{Right:} EM and action accuracy on the IDD,
  app-unseen, category-unseen, and task-unseen subsplits of AndroidControl low-level.}
  \label{fig:latency_and_breakdown}
\end{figure}

\subsection{Ablation Study}

We ablate the three core components of \method{}---latent CoT slots, APLR parallel
refinement, and the Q-Former world-model head---using Qwen3-VL-4B-Instruct on
AndroidWorld.

\paragraph{Ablation analysis.}
\begin{wrapfigure}[14]{r}{0.50\linewidth}
  \vspace{-4mm}
  \centering
  \includegraphics[width=\linewidth]{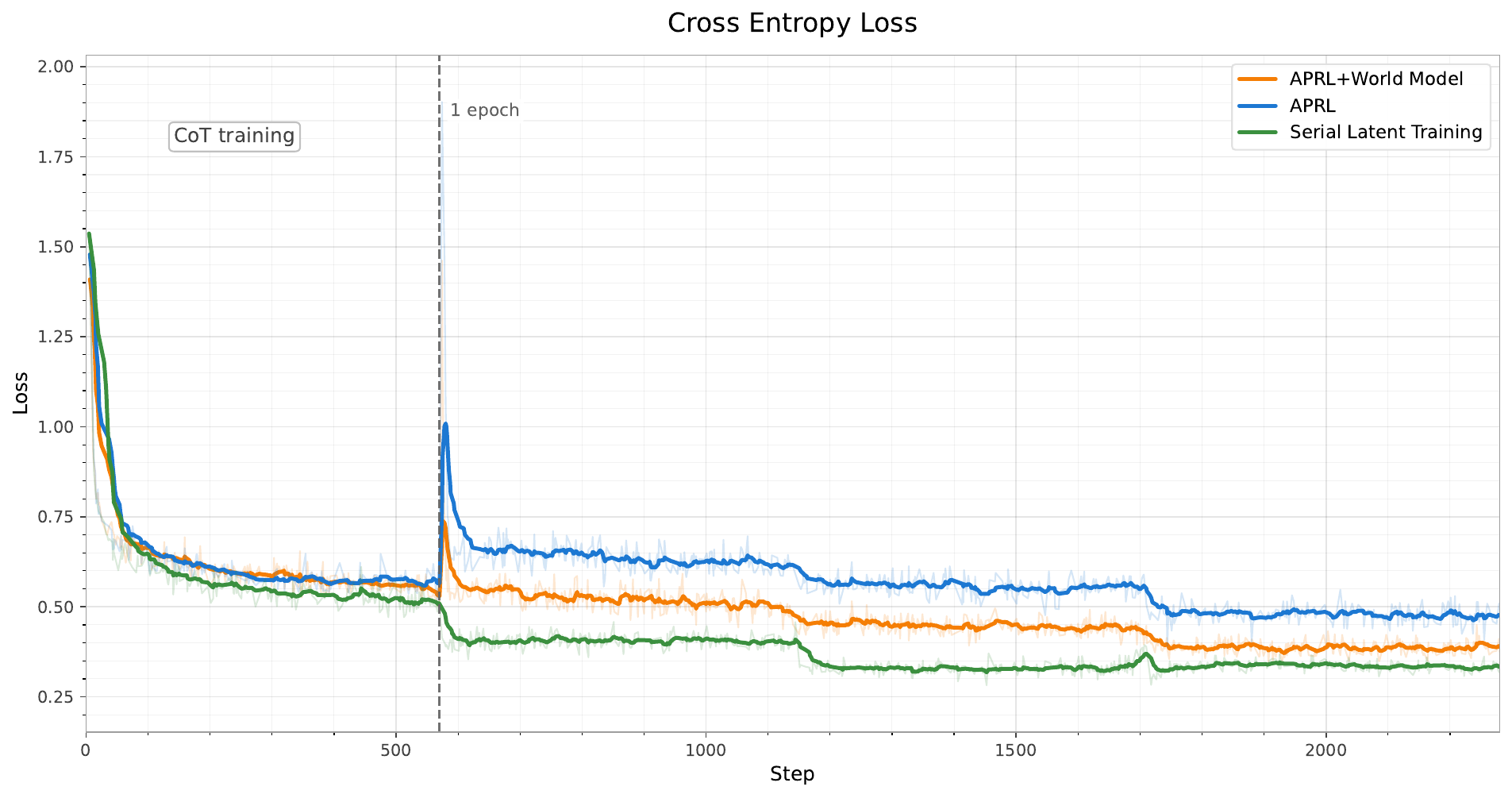}
  \vspace{-3mm}
  \caption{Cross-entropy training loss for the Qwen3-VL-4B variants listed in Table~\ref{tab:ablation_component}.}
  \label{fig:ablation_loss}
  \vspace{-4mm}
\end{wrapfigure}
Figure~\ref{fig:ablation_loss} reports cross-entropy loss for the same Qwen3-VL-4B
training variants summarized in Table~\ref{tab:ablation_component}. For the latent
settings, we allocate 9 latent slots; in the serial latent-CoT reference, the
observation, rationale, and prediction fields each correspond to three latent slots
and are masked through a stepwise curriculum. The sharp change after the first epoch
is expected rather than a failure mode: training switches from explicit-thought warmup
to latent-CoT training at that point, and the learning rate is reset with the new
stage. After this transition, APLR without the world-model objective exhibits a higher
loss than the serial latent-CoT reference, consistent with the fact that parallel
approximation leaves tail latent states less directly supervised by action CE alone.
Adding the Q-Former world-model objective makes the APLR curve track the serial
latent-CoT trend more closely, suggesting that next-frame feature alignment supplies
useful gradients for the transition-predictive parts of the latent state.

\newpage
\begin{wraptable}[12]{r}{0.50\linewidth}
  \vspace{-3mm}
  \caption{Component ablation on AndroidWorld (Qwen3-VL-4B, SR \%).}
  \label{tab:ablation_component}
  \centering
  \scriptsize
  \setlength{\tabcolsep}{2.2pt}
  \renewcommand{\arraystretch}{0.95}
  \resizebox{\linewidth}{!}{%
  \begin{tabular}{l c c c c}
    \toprule
    \textbf{Variant} & \textbf{Lat.} & \textbf{APLR} & \textbf{WM}
      & \textbf{SR} \\
    \midrule
    Base model        & \texttimes & \texttimes & \texttimes & 42.9 \\
    Action-only SFT   & \texttimes & \texttimes & \texttimes & 31.0 \\
    Explicit CoT SFT  & \texttimes & \texttimes & \texttimes & \textbf{52.6} \\
    Latent CoT, serial & \checkmark & \texttimes & \texttimes & 50.9 \\
    APLR only         & \checkmark & \checkmark & \texttimes & 48.2 \\
    \method{}-4B      & \checkmark & \checkmark & \checkmark & \textbf{52.6} \\
    \bottomrule
  \end{tabular}}
  \vspace{-4mm}
\end{wraptable}
Table~\ref{tab:ablation_component} shows the same pattern in final AndroidWorld SR.
Action-only SFT performs worse than the base model, indicating that removing thought
supervision without replacing it by internal computation can harm interactive
decision making. Explicit CoT SFT and full \method{}-4B both reach 52.6 SR, showing
that latent reasoning can match explicit CoT while avoiding decoded rationale tokens.
Serial latent CoT preserves much of this benefit (50.9 SR), and APLR without the
world-model objective reaches 48.2; adding the Q-Former world-model objective restores
the explicit-CoT-level result while keeping reasoning latent at inference time.

We study the sensitivity of \method{} to the number of latent slots $L$,
APLR refinement passes $K$, and the loss balance $\lambda_{\mathrm{ce}}$ on
AndroidWorld. Table~\ref{tab:ablation_efficiency} shows that the latent computation
budget matters.

\begin{table}[H]
  \caption{AndroidWorld ablations over latent-slot budget, APLR refinement passes,
  and loss balance. SR is task success rate (\%). Unless otherwise noted,
  $\lambda_{\mathrm{ce}}=0.8$.}
  \label{tab:ablation_efficiency}
  \centering
  \begin{tabular}{l c c c c}
    \toprule
    \textbf{Model} & \textbf{Latent slots} & \textbf{Ref.\ passes}
      & $\boldsymbol{\lambda_{\mathrm{ce}}}$ & \textbf{SR (\%)} \\
    \midrule
    \method{}-8B & 6 & 2 & 0.8 & 46.6 \\
    \method{}-8B & 6 & 3 & 0.8 & \textbf{57.8} \\
    \method{}-4B & 9 & 3 & 0.8 & 52.6 \\
    \method{}-4B & 3 & 3 & 0.8 & 32.8 \\
    \method{}-4B & 9 & 3 & 0.9 & 48.3 \\
    \bottomrule
  \end{tabular}
\end{table}

For \method{}-8B, increasing APLR refinement from two to three passes improves
AndroidWorld SR from 46.6 to 57.8, suggesting that the third pass substantially
reduces harmful tail-latent error. For \method{}-4B, reducing the latent budget from
9 to 3 slots lowers SR from 52.6 to 32.8, showing that the mobile-agent thought
state needs enough continuous capacity to represent observation, rationale, and
future-screen prediction. Finally, increasing $\lambda_{\mathrm{ce}}$ from 0.8 to 0.9
decreases SR from 52.6 to 48.3, indicating the value of retaining sufficient weight
on next-frame supervision.

\subsection{Latent Reasoning Visualization}
\label{sec:latent_viz}

We analyze a \method{}-8B checkpoint with six latent slots on the AndroidControl
IDD split; Appendix~\ref{app:latent_full_viz} provides full t-SNE and per-slot
action projections. Figure~\ref{fig:latent_viz} shows that latent training does not
collapse to a single undifferentiated representation. Before slot-centering, slots
1--2, 3--4, and 5--6 occupy three separated regions, corresponding to the
observation, rationale, and prediction dimensions of the structured thought. After
subtracting the slot mean, examples remain organized by action type: open-app, swipe,
type, and tap states form distinguishable regions in the centered latent space. These
visualizations indicate that \method{} learns both position-specific reasoning
subspaces and action-discriminative latent states.

\begin{figure}[t]
  \centering
  \begin{minipage}[t]{0.44\linewidth}
    \centering
    \includegraphics[width=0.92\linewidth]{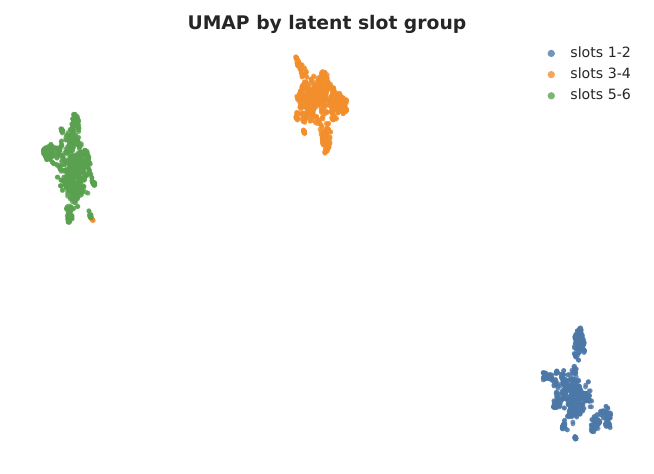}
  \end{minipage}\hfill
  \begin{minipage}[t]{0.48\linewidth}
    \centering
    \includegraphics[width=0.92\linewidth]{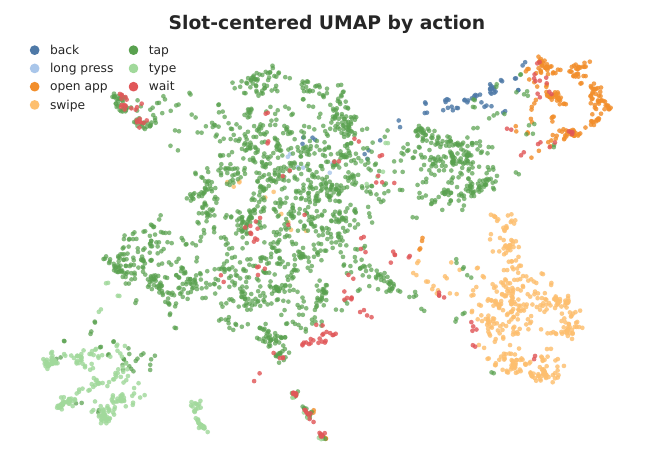}
  \end{minipage}
  \caption{\textbf{Left:} UMAP by latent slot group.
  \textbf{Right:} slot-centered UMAP by action type after subtracting per-slot means.}
  \label{fig:latent_viz}
\end{figure}

\FloatBarrier
\section{Limitations, Broader Impact, and Conclusion}

\method{} replaces visible rationales with APLR-refined latent slots regularized by
next-frame features, preserving explicit-CoT capability at much lower decoding cost.
Limitations include supervised-only training, feature-level world modeling,
next-frame supervision, and the need for privacy and action safeguards before deployment.

\newpage
\bibliographystyle{unsrtnat}
\bibliography{references}

@article{wei2022chain,
  title={Chain-of-thought prompting elicits reasoning in large language models},
  author={Wei, Jason and Wang, Xuezhi and Schuurmans, Dale and Bosma, Maarten and Xia, Fei and Chi, Ed and Le, Quoc V and Zhou, Denny and others},
  journal={Advances in neural information processing systems},
  volume={35},
  pages={24824--24837},
  year={2022}
}

@article{yao2022react,
  title={React: Synergizing reasoning and acting in language models},
  author={Yao, Shunyu and Zhao, Jeffrey and Yu, Dian and Du, Nan and Shafran, Izhak and Narasimhan, Karthik and Cao, Yuan},
  journal={arXiv preprint arXiv:2210.03629},
  year={2022}
}

@article{yao2022webshop,
  title={Webshop: Towards scalable real-world web interaction with grounded language agents},
  author={Yao, Shunyu and Chen, Howard and Yang, John and Narasimhan, Karthik},
  journal={Advances in Neural Information Processing Systems},
  volume={35},
  pages={20744--20757},
  year={2022}
}

@article{deng2023mind2web,
  title={Mind2web: Towards a generalist agent for the web},
  author={Deng, Xiang and Gu, Yu and Zheng, Boyuan and Chen, Shijie and Stevens, Sam and Wang, Boshi and Sun, Huan and Su, Yu},
  journal={Advances in Neural Information Processing Systems},
  volume={36},
  pages={28091--28114},
  year={2023}
}

@article{li2024effects,
  title={On the effects of data scale on ui control agents},
  author={Li, Wei and Bishop, William and Li, Alice and Rawles, Chris and Campbell-Ajala, Folawiyo and Tyamagundlu, Divya and Riva, Oriana},
  journal={Advances in Neural Information Processing Systems},
  volume={37},
  pages={92130--92154},
  year={2024}
}

@article{ha2018world,
  title={World models},
  author={Ha, David and Schmidhuber, J{\"u}rgen},
  journal={arXiv preprint arXiv:1803.10122},
  volume={2},
  number={3},
  pages={440},
  year={2018}
}

@article{hafner2019dream,
  title={Dream to control: Learning behaviors by latent imagination},
  author={Hafner, Danijar and Lillicrap, Timothy and Ba, Jimmy and Norouzi, Mohammad},
  journal={arXiv preprint arXiv:1912.01603},
  year={2019}
}

@inproceedings{li2023blip,
  title={Blip-2: Bootstrapping language-image pre-training with frozen image encoders and large language models},
  author={Li, Junnan and Li, Dongxu and Savarese, Silvio and Hoi, Steven},
  booktitle={International conference on machine learning},
  pages={19730--19742},
  year={2023},
  organization={PMLR}
}

@article{hao2022training,
  title={Training Large Language Models to Reason in a Continuous Latent Space},
  author={Hao, Shibo and Sukhbaatar, Sainbayar and Su, DiJia and Li, Xian and Hu, Zhiting and Weston, Jason and Tian, Yuandong},
  journal={arXiv preprint arXiv:2412.06769},
  year={2024}
}

@article{yuan2026fast,
  title={Fast-WAM: Do World Action Models Need Test-time Future Imagination?},
  author={Yuan, Tianyuan and Dong, Zibin and Liu, Yicheng and Zhao, Hang},
  journal={arXiv preprint arXiv:2603.16666},
  year={2026}
}

@article{rawles2024androidworld,
  title={Androidworld: A dynamic benchmarking environment for autonomous agents},
  author={Rawles, Christopher and Clinckemaillie, Sarah and Chang, Yifan and Waltz, Jonathan and Lau, Gabrielle and Fair, Marybeth and Li, Alice and Bishop, William and Li, Wei and Campbell-Ajala, Folawiyo and others},
  journal={arXiv preprint arXiv:2405.14573},
  year={2024}
}

@inproceedings{cheng2024seeclick,
  title={Seeclick: Harnessing gui grounding for advanced visual gui agents},
  author={Cheng, Kanzhi and Sun, Qiushi and Chu, Yougang and Xu, Fangzhi and YanTao, Li and Zhang, Jianbing and Wu, Zhiyong},
  booktitle={Proceedings of the 62nd Annual Meeting of the Association for Computational Linguistics (Volume 1: Long Papers)},
  pages={9313--9332},
  year={2024}
}

@article{wu2024atlas,
  title={Os-atlas: A foundation action model for generalist gui agents},
  author={Wu, Zhiyong and Wu, Zhenyu and Xu, Fangzhi and Wang, Yian and Sun, Qiushi and Jia, Chengyou and Cheng, Kanzhi and Ding, Zichen and Chen, Liheng and Liang, Paul Pu and others},
  journal={arXiv preprint arXiv:2410.23218},
  year={2024}
}

@article{qin2025ui,
  title={Ui-tars: Pioneering automated gui interaction with native agents},
  author={Qin, Yujia and Ye, Yining and Fang, Junjie and Wang, Haoming and Liang, Shihao and Tian, Shizuo and Zhang, Junda and Li, Jiahao and Li, Yunxin and Huang, Shijue and others},
  journal={arXiv preprint arXiv:2501.12326},
  year={2025}
}

@inproceedings{lin2025showui,
  title={Showui: One vision-language-action model for gui visual agent},
  author={Lin, Kevin Qinghong and Li, Linjie and Gao, Difei and Yang, Zhengyuan and Wu, Shiwei and Bai, Zechen and Lei, Stan Weixian and Wang, Lijuan and Shou, Mike Zheng},
  booktitle={Proceedings of the Computer Vision and Pattern Recognition Conference},
  pages={19498--19508},
  year={2025}
}

@article{gu2025ui,
  title={Ui-venus technical report: Building high-performance ui agents with rft},
  author={Gu, Zhangxuan and Zeng, Zhengwen and Xu, Zhenyu and Zhou, Xingran and Shen, Shuheng and Liu, Yunfei and Zhou, Beitong and Meng, Changhua and Xia, Tianyu and Chen, Weizhi and others},
  journal={arXiv preprint arXiv:2508.10833},
  year={2025}
}

@article{zhou2025mai,
  title={MAI-UI Technical Report: Real-World Centric Foundation GUI Agents},
  author={Zhou, Hanzhang and Zhang, Xu and Tong, Panrong and Zhang, Jianan and Chen, Liangyu and Kong, Quyu and Cai, Chenglin and Liu, Chen and Wang, Yue and Zhou, Jingren and others},
  journal={arXiv preprint arXiv:2512.22047},
  year={2025}
}

@article{wang2024mobile,
  title={Mobile-agent: Autonomous multi-modal mobile device agent with visual perception},
  author={Wang, Junyang and Xu, Haiyang and Ye, Jiabo and Yan, Ming and Shen, Weizhou and Zhang, Ji and Huang, Fei and Sang, Jitao},
  journal={arXiv preprint arXiv:2401.16158},
  year={2024}
}

@article{goyal2023think,
  title={Think before you speak: Training language models with pause tokens},
  author={Goyal, Sachin and Ji, Ziwei and Rawat, Ankit Singh and Menon, Aditya Krishna and Kumar, Sanjiv and Nagarajan, Vaishnavh},
  journal={arXiv preprint arXiv:2310.02226},
  year={2023}
}

@article{zelikman2024quiet,
  title={Quiet-star: Language models can teach themselves to think before speaking},
  author={Zelikman, Eric and Harik, Georges and Shao, Yijia and Jayasiri, Varuna and Haber, Nick and Goodman, Noah D},
  journal={arXiv preprint arXiv:2403.09629},
  year={2024}
}

@article{deng2023implicit,
  title={Implicit chain of thought reasoning via knowledge distillation},
  author={Deng, Yuntian and Prasad, Kiran and Fernandez, Roland and Smolensky, Paul and Chaudhary, Vishrav and Shieber, Stuart},
  journal={arXiv preprint arXiv:2311.01460},
  year={2023}
}

@inproceedings{assran2023self,
  title={Self-supervised learning from images with a joint-embedding predictive architecture},
  author={Assran, Mahmoud and Duval, Quentin and Misra, Ishan and Bojanowski, Piotr and Vincent, Pascal and Rabbat, Michael and LeCun, Yann and Ballas, Nicolas},
  booktitle={Proceedings of the IEEE/CVF conference on computer vision and pattern recognition},
  pages={15619--15629},
  year={2023}
}

@article{bardes2024revisiting,
  title={Revisiting feature prediction for learning visual representations from video},
  author={Bardes, Adrien and Garrido, Quentin and Ponce, Jean and Chen, Xinlei and Rabbat, Michael and LeCun, Yann and Assran, Mahmoud and Ballas, Nicolas},
  journal={arXiv preprint arXiv:2404.08471},
  year={2024}
}

@article{luo2025vimo,
  title={Vimo: A generative visual gui world model for app agents},
  author={Luo, Dezhao and Tang, Bohan and Li, Kang and Papoudakis, Georgios and Song, Jifei and Gong, Shaogang and Hao, Jianye and Wang, Jun and Shao, Kun},
  journal={arXiv preprint arXiv:2504.13936},
  year={2025}
}

@article{cao2026mobiledreamer,
  title={MobileDreamer: Generative Sketch World Model for GUI Agent},
  author={Cao, Yilin and Zhong, Yufeng and Zeng, Zhixiong and Zheng, Liming and Huang, Jing and Qiu, Haibo and Shi, Peng and Mao, Wenji and Guanglu, Wan},
  journal={arXiv preprint arXiv:2601.04035},
  year={2026}
}

@article{bai2025qwen3,
  title={Qwen3-vl technical report},
  author={Bai, Shuai and Cai, Yuxuan and Chen, Ruizhe and Chen, Keqin and Chen, Xionghui and Cheng, Zesen and Deng, Lianghao and Ding, Wei and Gao, Chang and Ge, Chunjiang and others},
  journal={arXiv preprint arXiv:2511.21631},
  year={2025}
}

@inproceedings{chai2025amex,
  title={Amex: Android multi-annotation expo dataset for mobile gui agents},
  author={Chai, Yuxiang and Huang, Siyuan and Niu, Yazhe and Xiao, Han and Liu, Liang and Wang, Guozhi and Zhang, Dingyu and Ren, Shuai and Li, Hongsheng},
  booktitle={Findings of the Association for Computational Linguistics: ACL 2025},
  pages={2138--2156},
  year={2025}
}

@article{hurst2024gpt,
  title={Gpt-4o system card},
  author={Hurst, Aaron and Lerer, Adam and Goucher, Adam P and Perelman, Adam and Ramesh, Aditya and Clark, Aidan and Ostrow, AJ and Welihinda, Akila and Hayes, Alan and Radford, Alec and others},
  journal={arXiv preprint arXiv:2410.21276},
  year={2024}
}

@article{luo2025gui,
  title={Gui-r1: A generalist r1-style vision-language action model for gui agents},
  author={Luo, Run and Wang, Lu and He, Wanwei and Chen, Longze and Li, Jiaming and Xia, Xiaobo},
  journal={arXiv preprint arXiv:2504.10458},
  year={2025}
}

@inproceedings{lu2026ui,
  title={Ui-r1: Enhancing efficient action prediction of gui agents by reinforcement learning},
  author={Lu, Zhengxi and Chai, Yuxiang and Guo, Yaxuan and Yin, Xi and Liu, Liang and Wang, Hao and Xiao, Han and Ren, Shuai and Zhao, Pengxiang and Liu, Guangyi and others},
  booktitle={Proceedings of the AAAI Conference on Artificial Intelligence},
  volume={40},
  number={21},
  pages={17608--17616},
  year={2026}
}

@article{yang2025ferret,
  title={Ferret-ui lite: Lessons from building small on-device gui agents},
  author={Yang, Zhen and Dou, Zi-Yi and Feng, Di and Huang, Forrest and Nguyen, Anh and You, Keen and Attia, Omar and Yang, Yuhao and Feng, Michael and Zhang, Haotian and others},
  journal={arXiv preprint arXiv:2509.26539},
  year={2025}
}

@article{hu2025tinyalign,
  title={TinyAlign: Boosting Lightweight Vision-Language Models by Mitigating Modal Alignment Bottlenecks},
  author={Hu, Yuanze and Fan, Zhaoxin and Wang, Xinyu and Li, Gen and Qiu, Ye and Yang, Zhichao and Wu, Wenjun and Wu, Kejian and Sun, Yifan and Deng, Xiaotie and others},
  journal={arXiv preprint arXiv:2505.12884},
  year={2025}
}

\appendix

\section{Additional Method Details}

\paragraph{Q-Former target construction.}
The world-model target is built online during training. For each sample with a valid next frame, the collator loads $x_{t+1}$ and the base VLM vision stack extracts post-spatial-merge patch features. The Q-Former aligner never receives these target features as input; it receives only latent CoT hidden states and row/column query embeddings.

\paragraph{No target leakage.}
The world-model predictor receives only the final latent hidden states from the current observation and history. Next-frame pixels are used only on the detached target side of the feature-alignment loss, and missing next frames are masked out.

\section{Training Data and Schedule}
\label{app:training_details}

\paragraph{Training data.}
We train on two sources of mobile-interaction data. First, we sample a subset of
\textit{AMEX}~\citep{chai2025amex}, the Android Multi-annotation Expo dataset,
which provides 104K high-resolution screenshots annotated with GUI-element
groundings and step-wise action chains across diverse Android applications. Second,
we collect \textit{self-explored trajectories} on AndroidWorld~\citep{rawles2024androidworld}
by running an exploration policy on device, producing task rollouts that extend the
distribution of observed app states beyond the static AMEX corpus.

\paragraph{Training procedure.}
Training proceeds in two stages. In Stage~1, we fine-tune the backbone for 1 epoch on
explicit chain-of-thought demonstrations, initializing the model with standard
text-based reasoning traces and action formatting. In Stage~2, we switch to the
latent-CoT objective and train for 3 epochs. The main configuration compresses
reasoning into $L=6$ latent slots and refines them over $K=3$ APLR passes; component
ablations additionally evaluate the 4B model with $L=9$ latent slots. Unless otherwise
specified, the cross-entropy loss weight is $\lambda_{\mathrm{ce}}=0.8$.

\section{Complete Latent Visualization Diagnostics}
\label{app:latent_full_viz}

This appendix reports the full latent-slot diagnostic views used to support
Section~\ref{sec:latent_viz}. The analysis is run on a \method{}-8B checkpoint evaluated on the
AndroidControl IDD split. The diagnostic dump contains 1,012 examples and 3,036
latent vectors after grouping the six latent slots into three adjacent pairs:
slots 1--2, slots 3--4, and slots 5--6. This grouping matches the main-paper
visualization, where the paired slots form three separated positional clusters.

\begin{figure}[p]
  \centering
  \includegraphics[width=\linewidth]{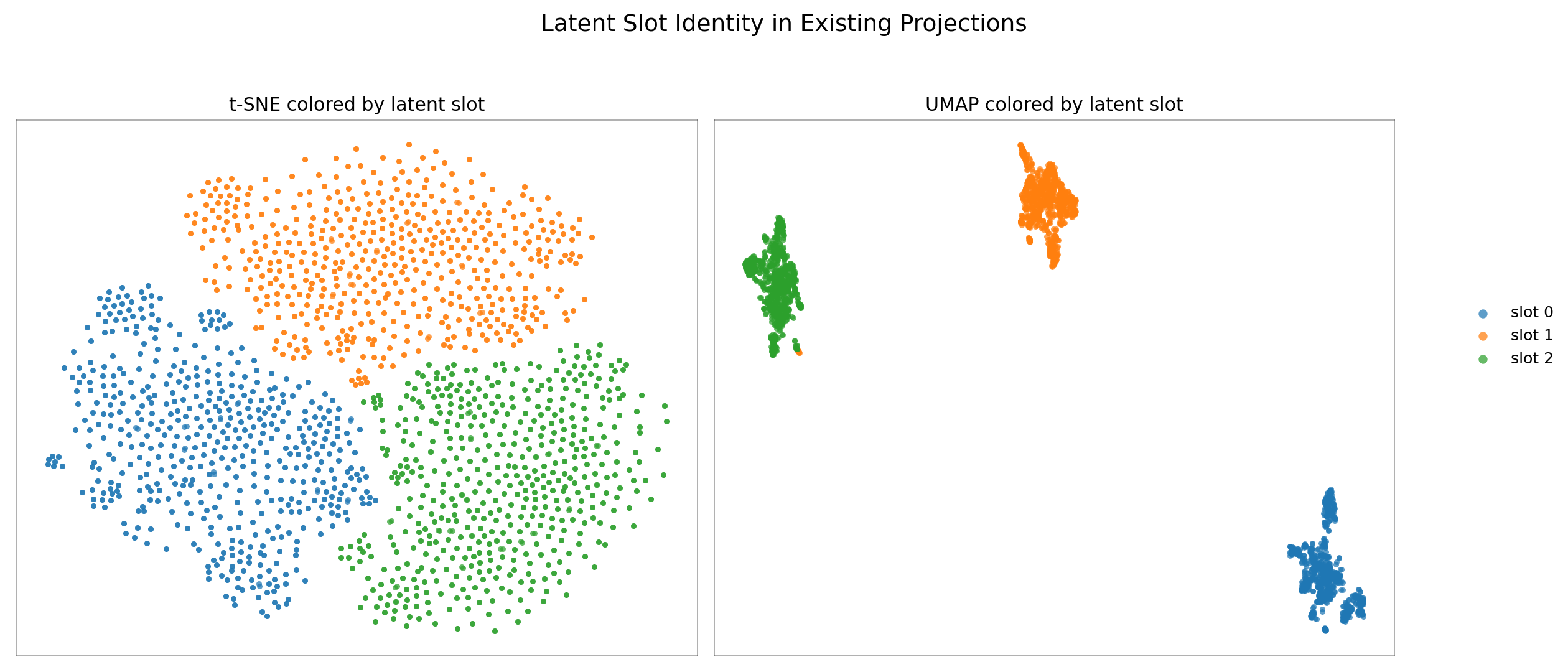}
  \caption{Latent slot identity in existing projections. The left panel colors a
  t-SNE projection by latent slot group, while the right panel colors a UMAP
  projection by the same group. Both projections show that adjacent slot pairs
  form three well-separated regions. This indicates that the latent representation
  contains a strong slot-position component: early, middle, and late latent slots
  occupy distinct subspaces before any slot-centering operation is applied.}
  \label{fig:app_slot_identity}
\end{figure}

Figure~\ref{fig:app_slot_identity} supports the interpretation that the latent
sequence is not a bag of exchangeable hidden states. The three slot groups form
separate clusters, especially in UMAP, which means the model uses different parts
of the continuous latent sequence differently. In the MIRAGE design, this is
desirable: the first slot group can specialize toward observation-level evidence,
the middle group toward action rationale, and the final group toward transition
prediction. The visualization does not prove this semantic assignment by itself,
but it shows that the model has learned a stable positional organization on which
such specialization can be built.

\begin{figure}[p]
  \centering
  \includegraphics[width=\linewidth]{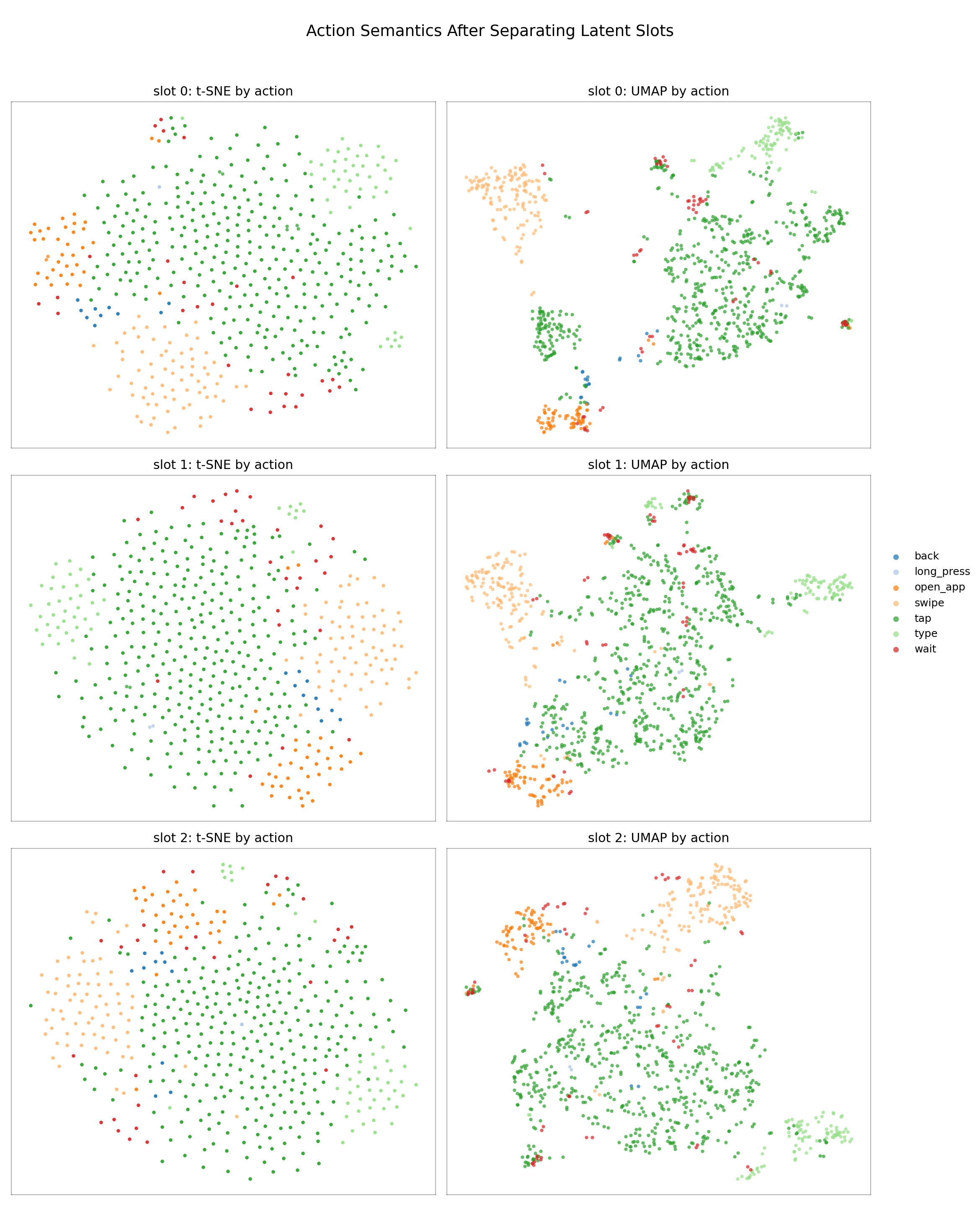}
  \caption{Per-slot action projections. Each subplot projects one slot group and
  colors points by the ground-truth action type. These plots ask whether action
  semantics are visible inside each slot group without first removing the slot
  mean.}
  \label{fig:app_per_slot_action}
\end{figure}

Figure~\ref{fig:app_per_slot_action} shows that action information is present but
partially entangled with slot identity. High-frequency tap examples spread broadly,
as expected for generic GUI interaction, while more specialized actions such as
open-app, swipe, and type appear in more localized regions. The per-slot views are
useful because they reveal whether one slot group alone carries the action signal.
In our diagnostic run, action separation is visible in each group but remains
imperfect, suggesting that action semantics are distributed across the latent
sequence rather than isolated in a single slot.

\begin{figure}[p]
  \centering
  \includegraphics[width=\linewidth]{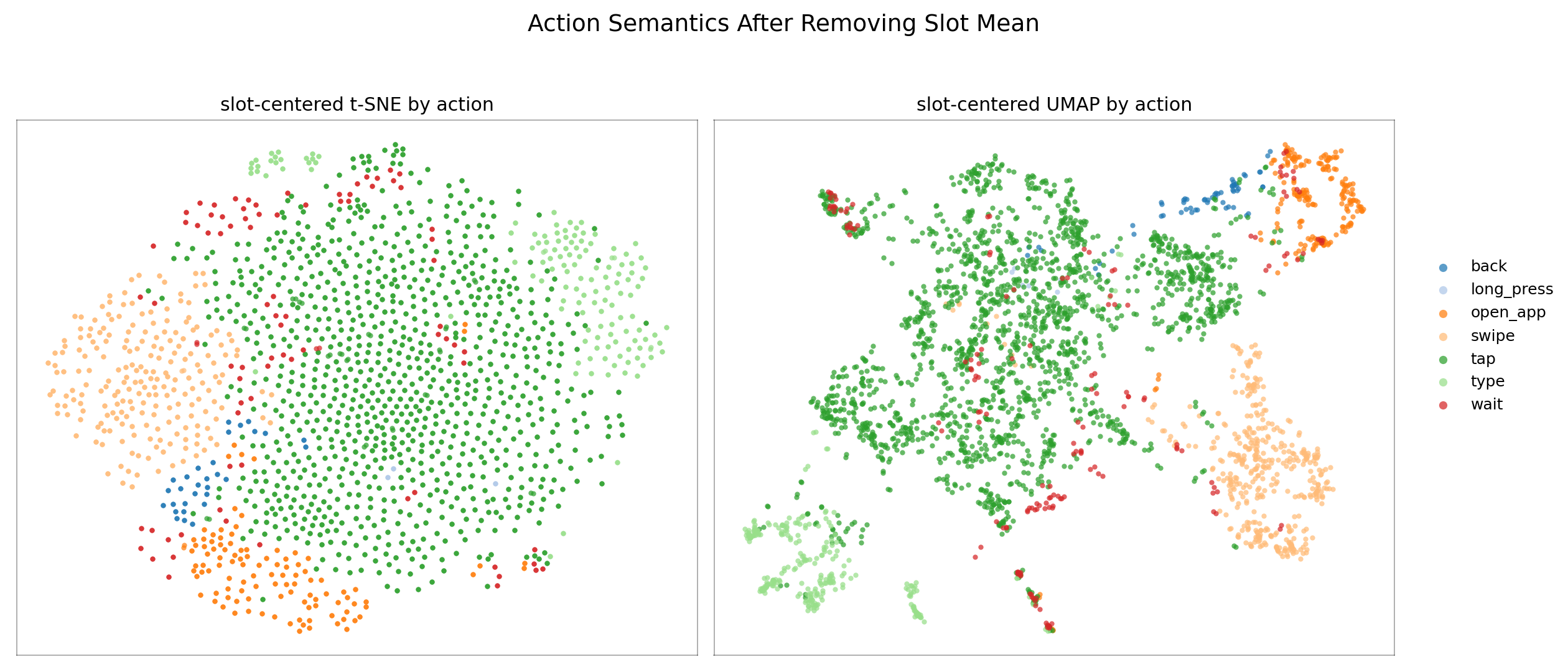}
  \caption{Action semantics after removing the slot mean. The left panel shows
  slot-centered t-SNE by action type, and the right panel shows slot-centered UMAP
  by action type. Slot centering subtracts the average representation of each
  slot group before projection, reducing the dominant slot-identity component.}
  \label{fig:app_slot_centered_action}
\end{figure}

Figure~\ref{fig:app_slot_centered_action} complements
Figure~\ref{fig:app_slot_identity}. Once the slot-specific mean is removed, action
structure becomes clearer. Open-app examples form a compact region, swipe examples
occupy a separate region, type examples gather in another area, and tap remains a
broad central manifold. This behavior is consistent with the main claim of
Section~4.3: MIRAGE latent states encode both where they are in the latent
reasoning sequence and what action semantics they support. The slot-position
component is strong, but it does not erase action-relevant information; after
centering, the residual representation still organizes examples by action type.

\section{Formal Derivation of Original Coconut}
\label{app:original_coconut}

This section formalizes the original Coconut-style latent reasoning procedure \citep{hao2022training} that motivates our APLR approximation. Coconut starts from supervised examples with a context $c$ and a visible chain of reasoning split into $M$ ordered chunks $\rho_{1:M}$, followed by a final answer or action $a$. For mobile agents, $c$ corresponds to the current screenshot, instruction, and interaction history, while $a$ is the serialized GUI action.

\paragraph{Explicit chain-of-thought objective.}
The ordinary explicit-CoT likelihood factorizes as
\begin{equation}
    p_\theta(\rho_{1:M},a\mid c)
    =
    \prod_{i=1}^{M}
    p_\theta(\rho_i\mid c,\rho_{<i})
    \cdot
    p_\theta(a\mid c,\rho_{1:M}),
    \label{eq:explicit_coconut_appendix}
\end{equation}
where each factor over a chunk denotes the product of token probabilities inside that chunk. Stage~0 of Coconut is equivalent to minimizing the standard next-token loss on the full sequence $(c,\rho_{1:M},a)$.

\paragraph{Continuous-thought substitution.}
Coconut introduces boundary tokens $\langle\mathrm{bot}\rangle$ and $\langle\mathrm{eot}\rangle$ for a latent reasoning span. Let $H_\theta(S)$ be the final-layer hidden states produced by the decoder for a mixed sequence $S$ containing both token embeddings and continuous vectors. Let $F_\theta(S)$ denote the hidden state at the last position of $S$, optionally followed by a projection $\Pi$ into the input-embedding space:
\begin{equation}
    F_\theta(S) = \Pi\, H_\theta(S)_{|S|}.
    \label{eq:coconut_hidden_map}
\end{equation}
At curriculum stage $m$, Coconut replaces the first $m$ visible reasoning chunks by $m$ continuous thoughts $z_{1:m}$. These thoughts are generated recursively:
\begin{align}
    S_0 &= [c;\langle\mathrm{bot}\rangle], \\
    z_i &= F_\theta([S_0;z_1;\ldots;z_{i-1}]),
    \qquad i=1,\ldots,m .
    \label{eq:coconut_recursive_latents}
\end{align}
The resulting training sequence is
\begin{equation}
    \tilde{S}^{(m)}
    =
    [c;\langle\mathrm{bot}\rangle;z_1;\ldots;z_m;\langle\mathrm{eot}\rangle;
    \rho_{m+1};\ldots;\rho_M;a].
    \label{eq:coconut_stage_sequence}
\end{equation}
Continuous positions are not vocabulary targets. The supervised likelihood at stage $m$ is therefore
\begin{equation}
    p_\theta^{(m)}(\rho_{m+1:M},a\mid c)
    =
    \prod_{i=m+1}^{M}
    p_\theta(\rho_i\mid c,z_{1:m},\rho_{m+1:i-1})
    \cdot
    p_\theta(a\mid c,z_{1:m},\rho_{m+1:M}),
    \label{eq:coconut_stage_likelihood}
\end{equation}
with loss
\begin{equation}
    \mathcal{L}_{\mathrm{Coconut}}^{(m)}
    =
    -\sum_{\ell\in\mathcal{I}^{(m)}}
    \log p_\theta(\tilde{S}^{(m)}_\ell\mid \tilde{S}^{(m)}_{<\ell}),
    \label{eq:coconut_stage_loss}
\end{equation}
where $\mathcal{I}^{(m)}$ indexes only discrete supervised tokens after the latent prefix. Increasing $m$ gradually moves supervision from explicit rationales to latent thoughts: $m=0$ recovers standard CoT training, while $m=M$ leaves only the final answer or action as visible supervision after the latent span.

\paragraph{Serial triangular computation.}
Equation~\eqref{eq:coconut_recursive_latents} is inherently serial because $z_i$ must be computed and inserted before $z_{i+1}$ can attend to it. Abstracting away the visible suffix, define a decoder-induced map $G_i$ for the $i$-th latent update under fixed context $c$:
\begin{equation}
    z_i = G_i(z_1,\ldots,z_{i-1};c),
    \qquad i=1,\ldots,m.
    \label{eq:coconut_triangular_system}
\end{equation}
Causal attention makes this system strictly triangular: $G_i$ cannot depend on $z_i$ or any future latent $z_j$ with $j>i$. Original Coconut solves Eq.~\eqref{eq:coconut_triangular_system} by forward substitution, i.e., it computes $z_1$, writes it into the sequence, then computes $z_2$, and so on. This exact serial execution is the reference solution that APLR approximates with parallel Jacobi-style refinement.

\section{Appendix Pipeline Pseudocode}
\label{app:pipeline_algorithm}

Algorithm~\ref{alg:mirage_pipeline} summarizes the full training pipeline.

\begin{algorithm}[t]
\caption{\method{} two-stage training pipeline}
\label{alg:mirage_pipeline}
\begin{algorithmic}[1]
\Require explicit records $(o_t,\tau_t,d_t,a_t)$ and sequential records with optional next frame $x_{t+1}$
\Ensure a latent-reasoning GUI policy; discard the Q-Former head at inference time

\Statex \textbf{Stage 1: explicit-thought warmup}
\State Serialize $[o_t;\langle\mathrm{THOUGHT}\rangle;\tau_t;\langle\mathrm{ACTION\_DESC}\rangle;d_t;\langle\mathrm{ACTION}\rangle;a_t]$
\State Optimize next-token CE over the structured thought, action description, and action tokens

\Statex \textbf{Stage 2: latent CoT and world-model training}
\State Replace the \texttt{<THOUGHT>} block with $N$ learned latent slots as in Eq.~\eqref{eq:latent_sequence}

\For{$r=0$ \textbf{to} $K-1$}
    \State Run one causal VLM pass and update all latent slots synchronously using Eq.~\eqref{eq:aplr_update}
\EndFor

\State Decode action tokens and compute $\mathcal{L}_{\mathrm{ce}}$ by Eq.~\eqref{eq:action_ce}

\If{a valid next frame $x_{t+1}$ exists}
    \State Extract detached targets $V_{t+1}=\mathrm{sg}(f_{\mathrm{vis}}(x_{t+1}))$
    \State Predict $\hat{V}_{t+1}$ from latent states with the Q-Former and compute $\mathcal{L}_{\mathrm{wm}}$
\Else
    \State Mask out the world-model term for this sample
\EndIf

\State Update parameters with $\mathcal{L}=\lambda\mathcal{L}_{\mathrm{ce}}+(1-\lambda)\mathcal{L}_{\mathrm{wm}}$
\State \Return latent-reasoning GUI policy
\end{algorithmic}
\end{algorithm}

\section{APLR Recovery Proof}
\label{app:aplr_proof}

We formalize the statement used in Section~\ref{sec:aplr}. The goal is not to prove that a small number of APLR rounds exactly reproduces every serial latent state. Instead, the result is sharper: after $K$ refinement rounds, the first $K$ latent states are exactly identical to the serial rollout, while the remaining latent states have a structured tail error that propagates only through the strictly causal dependency graph.

\paragraph{Setup.}
Consider one sequence with $N$ latent slots. Each latent state lies in $\mathbb{R}^d$. Let $c$ denote all non-latent context: instruction tokens, history tokens, image embeddings, attention masks, and position encodings. For each latent index $i$, define $G_i$ as the decoder-induced update that maps the fixed context and all earlier latent states to the hidden state used to replace latent slot $i$. Causal attention implies
\begin{equation}
    G_i = G_i(z_1,\ldots,z_{i-1}; c),
    \label{eq:gi_causal}
\end{equation}
so $G_i$ cannot depend on $z_i$ or any future latent $z_j$ with $j>i$.

The serial latent rollout is the forward-substitution solution
\begin{equation}
    s_i = G_i(s_1,\ldots,s_{i-1}; c), \qquad i=1,\ldots,N,
    \label{eq:serial_appendix}
\end{equation}
with the convention that $G_1$ depends only on $c$. APLR initializes every latent slot with the same learned latent embedding, written abstractly as $z_i^{(0)}=e_{\mathrm{lat}}$, and performs the parallel Jacobi-style update
\begin{equation}
    z_i^{(r+1)} = G_i(z_1^{(r)},\ldots,z_{i-1}^{(r)}; c),
    \qquad i=1,\ldots,N.
    \label{eq:aplr_appendix}
\end{equation}
The word ``parallel'' means that all right-hand sides in Eq.~\ref{eq:aplr_appendix} are evaluated using round-$r$ latent states before any slot is overwritten for round $r+1$.

\begin{lemma}[Causal dependency cone]
After $r$ APLR refinement rounds, latent slot $z_i^{(r)}$ can depend on the initial latent values only through slots with indices at most $i-r$. Equivalently, information from the exact serial prefix can advance by at most one latent position per refinement round.
\end{lemma}

\begin{proof}
The claim is immediate for $r=0$ because $z_i^{(0)}$ is itself an initial value. Assume it holds for round $r$. At round $r+1$, $z_i^{(r+1)}$ is a function of $z_1^{(r)},\ldots,z_{i-1}^{(r)}$ and the fixed context $c$. By the induction hypothesis, each predecessor $z_j^{(r)}$ can depend on initial slots only up to index $j-r$. Since $j\leq i-1$, the largest such index is $(i-1)-r=i-(r+1)$. Thus $z_i^{(r+1)}$ can depend on initial values only through slots with indices at most $i-(r+1)$. This proves the claim by induction.
\end{proof}

\begin{proposition}[Exact recovery of early latent states]
For any number of refinement rounds $r \geq 0$, APLR satisfies
\begin{equation}
    z_i^{(r)} = s_i \qquad \text{for all } i \leq r.
\end{equation}
Consequently, after $K$ refinement rounds, APLR exactly recovers the first $K$ latent states of the serial rollout, and after $N$ rounds it recovers the full serial latent sequence.
\end{proposition}

\begin{proof}
We prove the claim by induction on the refinement round $r$.

\emph{Base case.} For $r=0$, the set of indices satisfying $i \leq 0$ is empty, so the claim holds vacuously. No latent has been refined yet, and no equality with the serial rollout is required.

\emph{Inductive step.} Assume that after round $r$, $z_j^{(r)}=s_j$ for every $j \leq r$. Consider round $r+1$ and any latent index $i \leq r+1$. If $i=1$, then the causal update has no latent predecessor. Both APLR and serial forward substitution therefore use exactly the same context-only map:
\begin{equation}
    z_1^{(r+1)} = G_1(c) = s_1.
\end{equation}
If $i>1$, then every predecessor index $j<i$ satisfies $j \leq r$ because $i\leq r+1$. By the induction hypothesis, all predecessor values used by the APLR update are already equal to their serial values:
\begin{equation}
    (z_1^{(r)},\ldots,z_{i-1}^{(r)})
    =
    (s_1,\ldots,s_{i-1}).
\end{equation}
Substituting these equal predecessors into the parallel update gives
\begin{equation}
    z_i^{(r+1)}
    = G_i(z_1^{(r)},\ldots,z_{i-1}^{(r)}; c)
    = G_i(s_1,\ldots,s_{i-1}; c)
    = s_i,
\end{equation}
where the final equality follows from the serial recurrence in Eq.~\ref{eq:serial_appendix}. Thus the claim holds for round $r+1$, completing the induction.
\end{proof}

\begin{corollary}[Which latent slots remain approximate]
After $K$ refinement rounds, the only latent slots that can differ from the serial rollout are the tail slots
\begin{equation}
    z_{K+1}^{(K)}, z_{K+2}^{(K)}, \ldots, z_N^{(K)}.
\end{equation}
The first $K$ slots have zero serial-approximation error.
\end{corollary}

\paragraph{Local tail-error expression.}
The same triangular structure explains the approximation error for the tail slots when $K<N$. Let $\delta_i^{(r)} = z_i^{(r)} - s_i$. Assume each $G_i$ is twice continuously differentiable in a neighborhood of the serial solution. Define the block Jacobian
\begin{equation}
    A_{ij} =
    \begin{cases}
    \frac{\partial G_i}{\partial z_j}\big|_{s,c}, & j<i,\\
    0, & j\geq i.
    \end{cases}
\end{equation}
The matrix $A$ is strictly block-lower-triangular. A Taylor expansion around $(s_1,\ldots,s_{i-1})$ gives
\begin{equation}
    \delta_i^{(r+1)}
    =
    \sum_{j<i}
    A_{ij}
    \delta_j^{(r)}
    + R_i^{(r)},
    \qquad
    \|R_i^{(r)}\| \leq C_i \|\delta_{<i}^{(r)}\|^2
\end{equation}
for constants $C_i$ determined by local Hessian bounds. 
Stacking all latent errors yields
\begin{equation}
    \delta^{(r+1)} = A\delta^{(r)} + R^{(r)},
    \qquad
    \|R^{(r)}\| = O(\|\delta^{(r)}\|^2).
\end{equation}
Ignoring higher-order terms, the $K$-round tail error is
\begin{equation}
    \delta^{(K)} \approx A^K \delta^{(0)}.
    \label{eq:tail_error_linear}
\end{equation}
Because $A$ is strictly lower triangular, $(A^K)_{ij}$ can be nonzero only when $j \leq i-K$. Thus Eq.~\ref{eq:tail_error_linear} has two consequences. First, rows $i\leq K$ are exactly zero, matching the exact-recovery proposition. Second, for a tail latent $i>K$, the remaining error comes only from causal chains of length at least $K$ that connect earlier imperfect initial values to slot $i$. These are precisely the deep latent states that have not received enough refinement rounds to match the serial rollout.

\section{How the World-Model Loss Regularizes Tail Errors}
\label{app:wm_tail_error}

The previous section identifies the approximation error left by using $K<N$ APLR rounds: it lives only in the tail latent slots $z_{K+1:N}^{(K)}$. The Q-Former world-model loss does not make these tail slots algebraically identical to the serial rollout. Instead, it adds direct supervision to the components of their error that affect prediction of the next mobile screen in the VLM visual feature space. This is the useful notion of compensation for a mobile agent: the auxiliary loss penalizes tail errors that discard environment-transition information.

\paragraph{Tail-error notation.}
Let
\begin{equation}
    \tau^{(K)} =
    \big[
    \delta_{K+1}^{(K)},\ldots,\delta_N^{(K)}
    \big]
\end{equation}
denote the stacked tail error after $K$ APLR rounds. Let $P_\phi$ denote the Q-Former aligner, including the row/column query embeddings, cross-attention into latent CoT hidden states, and final visual projection. For a valid next-frame transition, define the detached target feature matrix
\begin{equation}
    V^\star = \mathrm{sg}(f_{\mathrm{vis}}(x_{t+1})) \in \mathbb{R}^{M\times H_v}.
\end{equation}
The Q-Former prediction as a function of the latent state is
\begin{equation}
    \hat{V}(\tau) = P_\phi(s+\tau; R),
\end{equation}
where $R$ is the set of valid row/column patch coordinates. For the cosine variant used by default, write normalized predicted and target features as
\begin{equation}
    g(\tau) = \mathrm{norm}(\hat{V}(\tau)),
    \qquad
    u = \mathrm{norm}(V^\star),
\end{equation}
where normalization is applied patchwise. The masked world-model loss is
\begin{equation}
    \ell_{\mathrm{wm}}(\tau)
    =
    \frac{1}{|\mathcal{M}|}
    \sum_{j\in\mathcal{M}}
    \left(1-\langle g_j(\tau), u_j\rangle\right)
    =
    \frac{1}{2|\mathcal{M}|}
    \sum_{j\in\mathcal{M}}
    \|g_j(\tau)-u_j\|_2^2,
    \label{eq:cosine_squared_distance}
\end{equation}
because both $g_j$ and $u_j$ have unit norm. Here $\mathcal{M}$ contains only real next-frame patches from samples with a valid next frame and at least one latent slot.

\paragraph{Local expansion.}
Let $J$ be the Jacobian of the normalized Q-Former prediction $g(\tau)$ with respect to the tail error at $\tau=0$:
\begin{equation}
    g(\tau) = g(0) + J\tau + O(\|\tau\|^2).
    \label{eq:qformer_jacobian}
\end{equation}
Substituting Eq.~\ref{eq:qformer_jacobian} into Eq.~\ref{eq:cosine_squared_distance} yields
\begin{align}
    \ell_{\mathrm{wm}}(\tau) - \ell_{\mathrm{wm}}(0)
    &=
    \frac{1}{|\mathcal{M}|}
    \langle J^\top(g(0)-u), \tau\rangle
    +
    \frac{1}{2|\mathcal{M}|}
    \|J\tau\|_2^2
    +
    O(\|\tau\|^3).
    \label{eq:wm_second_order}
\end{align}
When the serial latent solution is locally optimized for the world-model target, or when we analyze directions orthogonal to the residual gradient $J^\top(g(0)-u)$, the linear term vanishes. The leading term is then a quadratic penalty on the components of $\tau$ that the Q-Former prediction can observe.

\begin{proposition}[Q-Former world-model supervision controls future-predictive tail error]
Let $U$ be a subspace of tail-error directions on which the normalized Q-Former Jacobian is locally observable: for all $\xi\in U$,
\begin{equation}
    \|J\xi\|_2 \geq \sigma_U \|\xi\|_2
\end{equation}
for some $\sigma_U>0$. Assume the residual-gradient term in Eq.~\ref{eq:wm_second_order} is zero or treated as a first-order optimization gradient. Then, for sufficiently small tail error,
\begin{equation}
    \ell_{\mathrm{wm}}(\tau^{(K)}) - \ell_{\mathrm{wm}}(0)
    \geq
    \frac{\sigma_U^2}{2|\mathcal{M}|}
    \|P_U \tau^{(K)}\|^2
    - O(\|\tau^{(K)}\|^3),
    \label{eq:qformer_tail_bound}
\end{equation}
where $P_U$ projects onto the future-predictive subspace of tail errors that change normalized next-frame feature predictions.
\end{proposition}

\begin{proof}
Using Eq.~\ref{eq:wm_second_order} and dropping the zero residual-gradient term gives
\begin{equation}
    \ell_{\mathrm{wm}}(\tau^{(K)}) - \ell_{\mathrm{wm}}(0)
    \geq
    \frac{1}{2|\mathcal{M}|}\|J\tau^{(K)}\|_2^2
    - O(\|\tau^{(K)}\|^3).
\end{equation}
Decompose the tail error into a future-predictive component and a locally unobserved component:
\begin{equation}
    \tau^{(K)} = P_U\tau^{(K)} + (I-P_U)\tau^{(K)}.
\end{equation}
By definition of $U$, the Q-Former prediction is informative on $P_U\tau^{(K)}$, with minimum singular value at least $\sigma_U$. Therefore
\begin{equation}
    \|J\tau^{(K)}\|_2
    \geq
    \sigma_U \|P_U\tau^{(K)}\|.
\end{equation}
Substituting this inequality into the previous display yields Eq.~\ref{eq:qformer_tail_bound}.
\end{proof}

\paragraph{MSE variant.}
If the implementation uses the MSE world-model loss instead of cosine distance, the same argument applies without patchwise normalization. Let $\hat{V}(\tau)=\hat{V}(0)+B\tau+O(\|\tau\|^2)$. The masked MSE loss has local expansion
\begin{equation}
    \ell_{\mathrm{mse}}(\tau)-\ell_{\mathrm{mse}}(0)
    =
    \frac{1}{|\mathcal{M}|}
    \langle B^\top(\hat{V}(0)-V^\star),\tau\rangle
    +
    \frac{1}{|\mathcal{M}|}\|B\tau\|_2^2
    +
    O(\|\tau\|^3).
\end{equation}
Thus any tail-error direction that changes predicted next-frame features is quadratically penalized once the first-order residual term is optimized.

\paragraph{Interpretation.}
The proposition says exactly which unrefined latent-token errors the Q-Former world-model objective can compensate: errors in the tail slots $z_{K+1:N}^{(K)}$ that are visible to next-frame feature prediction, i.e., errors outside the local nullspace of the Q-Former prediction Jacobian. These directions correspond to latent information that changes the predicted semantic layout of the next screenshot, such as whether a tap opens a menu, whether a typed query changes a text field, or whether navigation moves to a new screen. When the combined loss
\begin{equation}
    \mathcal{L}
    =
    \lambda_{\mathrm{ce}}\mathcal{L}_{\mathrm{ce}}
    +
    \lambda_{\mathrm{wm}}\mathcal{L}_{\mathrm{wm}}
\end{equation}
is optimized, any such future-predictive tail error increases the objective by at least
\begin{equation}
    \frac{\lambda_{\mathrm{wm}}\sigma_U^2}{2|\mathcal{M}|}
    \|P_U\tau^{(K)}\|^2
\end{equation}
up to higher-order terms. Thus the auxiliary world-model loss supplies curvature and gradients for tail latent states that may be weakly constrained by action imitation alone. It does not control tail-error directions that do not affect next-frame visual feature predictions; those directions are either irrelevant to the chosen environment-model target or must be constrained by the action loss and other regularizers.

\section{System Prompt}
\label{app:system_prompt}

Figure~\ref{fig:system_prompt} shows the complete system prompt provided to the
\method{} agent at inference time. The prompt defines the role, input format, output
format with structured \texttt{<THOUGHT>} / \texttt{<ACTION\_DESC>} / \texttt{<ACTION>}
tags, and the full command vocabulary. At inference time the
\texttt{<THOUGHT>} block is replaced by the $N$ latent slots described in
Section~\ref{sec:aplr}; the explicit text template here is therefore used only during
Stage~1 warmup training.

\begin{figure}[htbp]
\centering
\begin{tcolorbox}[
  breakable,
  enhanced,
  colback=gray!6,
  colframe=black!55,
  arc=4pt,
  boxrule=0.7pt,
  left=6pt, right=6pt, top=5pt, bottom=5pt,
  fontupper=\ttfamily\scriptsize,
  title={\small\sffamily\bfseries System Prompt (inference-time template)},
  attach boxed title to top left={yshift=-2mm, xshift=6pt},
  boxed title style={colback=black!55, arc=3pt, boxrule=0pt,
                     fontupper=\color{white}\small\sffamily\bfseries}
]
\begin{lstlisting}[style=systemprompt]
You are an Android mobile GUI agent.
Input: the user task, a summary list of completed past actions, and the current screenshot.
Choose the single best next action from the current screen only. Do not invent unseen UI.
Produce one atomic step. Coordinates are normalized integers in [0,999]. Do not emit
markdown fences or hidden reasoning tags such as <THOUGHT>.

Output:
- Emit only the required blocks using the literal tags `<THOUGHT>`, `<ACTION_DESC>`,
  and `<ACTION>`.
- `<THOUGHT>`: one English line with exactly
  `1. observation: ... | 2. rationale: ... | 3. predict: ...`.
  - `observation`: describe only the currently visible UI on the current screenshot.
  - `rationale`: combine task goal, prior progress, and the reason this single next
    action is best.
  - `predict`: describe the immediate visible result right after the action.
  Keep it grounded in the current screenshot and history. Do not mention coordinates.
  Do not mention prompt structure, annotation metadata, or screenshot-availability
  phrases such as "next screenshot", "no next screenshot", "one image", or "two images".
- `<ACTION_DESC>`: English, one imperative phrase.
- `<ACTION>`: exactly one executable operation command string.

If the task is completed successfully, use `finish(success)`. If the task should
terminate unsuccessfully, use `finish(fail)`.

Command format:
- Tap a visible UI element:
    click(<|point_start|>(x,y)<|point_end|>)
- Swipe to move the page or navigate across the screen:
    swipe(<|point_start|>(x1,y1)<|point_end|>, <|point_start|>(x2,y2)<|point_end|>)
- Drag to move a slider, handle, or draggable object:
    drag(<|point_start|>(x1,y1)<|point_end|>, <|point_start|>(x2,y2)<|point_end|>)
- Long press to open a context action or select an item:
    long_press(<|point_start|>(x,y)<|point_end|>)
- Type into the currently focused field:
    type("text")
- Type into a specific text box:
    type(<|box_start|>(x1,y1),(x2,y2)<|box_end|>, "text")
- Press a supported device key:
    press(HOME)
- Open an app directly when that is the intended system action:
    open_app("Settings")
- Wait briefly for the UI to load or update:
    wait()  or  wait(t)
- Return a textual answer for a non-GUI question:
    answer("text")
- Request human takeover when the task cannot be safely continued:
    call_user()
- End the task:
    finish(success)  or  finish(fail)

Use exactly one command. Do not wrap it in JSON or natural language. Use English double
quotes for string arguments and uppercase English for key names in press(...).
\end{lstlisting}
\end{tcolorbox}
\caption{Full system prompt used by \method{}. The \texttt{<THOUGHT>} block is a
  visible text template during Stage~1 warmup; it is replaced by learned latent slots
  during Stage~2 and at inference time. Coordinates are normalized integers in
  $[0, 999]$.}
\label{fig:system_prompt}
\end{figure}

\section{Action Space}
\label{app:action_space}

\method{} operates over a compact, typed action space that maps directly to Android
system APIs. Each action is a single atomic command; the agent never emits more than
one command per step. Table~\ref{tab:action_space} lists all supported action types.

\begin{table}[htbp]
\centering
\caption{Complete action space of \method{}. Coordinates $(x,y)$ are normalized
  integers in $[0,999]$ relative to the screenshot dimensions.}
\label{tab:action_space}
\small
\begin{tabular}{@{}lp{5.8cm}p{5.2cm}@{}}
\toprule
\textbf{Action type} & \textbf{Signature} & \textbf{Description} \\
\midrule
\texttt{click}      & \texttt{click((x,y))}
  & Tap a visible UI element at screen coordinate $(x,y)$. \\[3pt]
\texttt{swipe}      & \texttt{swipe((x1,y1),(x2,y2))}
  & Swipe from $(x_1,y_1)$ to $(x_2,y_2)$ to scroll or navigate. \\[3pt]
\texttt{drag}       & \texttt{drag((x1,y1),(x2,y2))}
  & Drag a slider, handle, or draggable element. \\[3pt]
\texttt{long\_press} & \texttt{long\_press((x,y))}
  & Long-press to open a context menu or select an item. \\[3pt]
\texttt{type}       & \texttt{type("text")} \newline
                      \texttt{type((x1,y1),(x2,y2), "text")}
  & Type into the focused field, or into a specific text box
    identified by its bounding box. \\[3pt]
\texttt{press}      & \texttt{press(KEY)}
  & Press a hardware key (\texttt{HOME}, \texttt{BACK}, \texttt{ENTER},
    \texttt{VOLUME\_UP}, etc.). \\[3pt]
\texttt{open\_app}  & \texttt{open\_app("AppName")}
  & Launch an installed application by name when that is the
    intended system action. \\[3pt]
\texttt{wait}       & \texttt{wait()} \;\;/\;\; \texttt{wait(t)}
  & Wait for the UI to load or update; optional timeout $t$
    in seconds. \\[3pt]
\texttt{answer}     & \texttt{answer("text")}
  & Return a textual answer for a non-GUI information-retrieval
    question without touching the screen. \\[3pt]
\texttt{call\_user} & \texttt{call\_user()}
  & Request human takeover when the task cannot be safely
    continued (e.g.\ requires biometric authentication). \\[3pt]
\texttt{finish}     & \texttt{finish(success)} \;\;/\;\; \texttt{finish(fail)}
  & Terminate the episode with a success or failure signal. \\
\bottomrule
\end{tabular}
\end{table}

\paragraph{Coordinate convention.}
All spatial coordinates are normalized to $[0, 999]$ independently in width and height,
mapping the top-left corner of the screenshot to $(0,0)$ and the bottom-right corner
to $(999, 999)$. This convention is device-resolution-agnostic: the same action is
valid on any screen size and is converted to absolute pixel coordinates at execution
time by the environment driver.

\paragraph{Action tokenization.}
Each action command is serialized as a plain ASCII string (e.g.\
\texttt{click((512,340))}) and appended to the VLM's output token sequence. The model
is trained with next-token cross-entropy over the action tokens
$\mathcal{L}_{\mathrm{ce}}$ (Eq.~\ref{eq:action_ce}). No structured JSON, markup, or
special delimiters are used; the action string is the only content inside the
\texttt{<ACTION>} block.

\end{document}